%% file: neurips_2023.tex
\documentclass{article}

\PassOptionsToPackage{numbers, compress}{natbib}

\usepackage[preprint]{neurips_2023}

\usepackage[utf8]{inputenc} 
\usepackage[T1]{fontenc}    
\usepackage[colorlinks, citecolor=blue]{hyperref}       
\usepackage{url}            
\usepackage{booktabs}       
\usepackage{amsfonts}       
\usepackage{nicefrac}       
\usepackage{microtype}      
\usepackage{xcolor}         
\input{myheaders}
\usepackage{xspace}
\title{A Mechanism for Sample-Efficient In-Context Learning for Sparse Retrieval Tasks}

\author{%
  Jacob Abernethy \\
  Google Research\\
  \& Georgia Institute of Technology \\
  \texttt{abernethyj@google.com} \\
   \And
   Alekh Agarwal \\
   Google Research \\
   \texttt{alekhagarwal@google.com} \\
   \And
   Teodor V. Marinov \\
   Google Research \\
   \texttt{tvmarinov@google.com} \\
   \And
   Manfred K. Warmuth \\
   Google Research \\
   \texttt{manfred@google.com} \\
}

\begin{document}

\maketitle

\begin{abstract}
  We study the phenomenon of \textit{in-context learning} (ICL) exhibited by large language models, where they can adapt to a new learning task, given a handful of labeled examples, without any explicit parameter optimization. Our goal is to explain how a pre-trained transformer model is able to perform ICL under reasonable assumptions on the pre-training process and the downstream tasks. We posit a mechanism whereby a transformer can achieve the following: (a) receive an i.i.d. sequence of examples which have been converted into a prompt using potentially-ambiguous delimiters, (b) correctly segment the prompt into examples and labels, 
  (c) infer from the data a \textit{sparse linear regressor} hypothesis, and finally (d) apply this hypothesis on the given test example and return a predicted label.
  We establish that this entire procedure is implementable using the transformer mechanism, and we give sample complexity guarantees for this learning framework.
  Our empirical findings validate the challenge of segmentation, and we show a correspondence between our posited mechanisms and observed attention maps for step (c).
\end{abstract}

\section{Introduction}
\label{sec:intro}

In-context learning has emerged as a powerful and novel paradigm where, starting with \citet{brown2020language}, it has been observed that a pre-trained language model can ``learn'' simply through prompting with a handful of desired input-output pairs from a new task. Strikingly, the model is able to perform well on future input queries from the same task by simply conditioning on this prompt, without updating any model parameters, and using a surprisingly small number of examples in the prompt to learn a target task. While model fine-tuning for few shot learning can be explained in terms of the vast literature on transfer learning and domain adaptation, ICL eludes an easy explanation for its sample-efficiency and versatility. In this paper, we study the question: \emph{What are plausible mechanisms to explain ICL for some representative tasks and what is their sample complexity?} 

Before discussing potential answers, we note why the ICL capability is surprising, and merits a careful study. Typical few-shot learning settings consist of a family of related tasks among which transfer is expected. On the other hand, the pre-training task of predicting the next token for language models appears largely disconnected from the variety of downstream tasks ranging from composing verses to answering math and analogy questions or writing code, that they are later prompted for. More importantly, a typical prompt consists of a sequence of unrelated inputs for a problem, followed by the desired outputs. This should constitute a very unlikely sequence for the model, since its training data seldom contains input, output pairs for a single concept occurring together. Due to this abruptness of example boundaries, recognizing such a prompt as a sequence of independent examples to learn across is an impressive feat in of itself. Once the model \emph{segments} this prompt, it still needs to learn a \emph{consistent hypothesis across examples} in the prompt to map inputs to desired outputs, and then apply this learned hypothesis to fresh query inputs. Given that this capability has been primarily observed in transformer based models, we investigate if there are aspects of self-attention which are particularly well-suited to addressing the aforementioned challenges?

In this paper, we study all the questions mentioned above. A formal investigation of ICL was pioneered in \citet{gargcan}, who trained transformer models~\citep{vaswani2017attention} from scratch that can learn linear regression via ICL. Within this model, gradient descent or closed form ridge regression approaches to map the prompt to a hypothesis were put forth in~\citet{von2022transformers} and~\citet{akyurek2022learning}. These works show that the posited mechanisms can be implemented using a transformer, given an appropriate formatting and tokenization of the inputs. \citet{dai2022can} study the relationship between linear attention and gradient descent, and \citet{li2023transformers} study transformers as producing general purpose learning algorithms. Of these, only \citet{li2023transformers} studies sample complexity aspects, though the stability parameter in their bounds is not explicitly quantified. From a statistical perspective, \citet{xie2021explanation} and \citet{zhang2022analysis} cast ICL as posterior inference, with the former studying a mixture of HMM models and the latter analyzing more general exchangeable sequences. \citet{wies2023learnability} give PAC guarantees for ICL, when pre-trained on a mixture of downstream tasks. These works do not, however, provide mechanisms to implement the desired learning procedures using a transformer. \citet{olsson2022context} give some evidence that ICL might arise from a mechanism called induction heads, but do not discuss the sample complexity aspects or describe how induction heads might be leveraged to address a variety of learning tasks. We defer a more detailed discussion relative to these works, as well as connections with the broader literature on uses of the ICL capability to Appendix~\ref{sec:related}.

\textbf{Our Contributions.} Our work studies the ICL process in an end-to-end manner. Unlike most prior works on ICL, which either require pre-training task to be identical to the downstream task~\citep{gargcan, von2022transformers, akyurek2022learning}, or comprised of some mixture of downstream tasks~\citep{xie2021explanation, wies2023learnability}, we abstract the details of this procedure by representing it as a fixed and given \emph{prior} distribution over sequences. Our results include:

\begin{itemize}[nosep, leftmargin=2pt]
\item \textbf{Prompt segmentation:} We propose a segmentation mechanism for the prompt, which maximizes the likelihood of a proposed segmentation under the prior learned during pre-training. The mechanism crucially leverages aspect of the attention architecture to \emph{learn the segmentation} with few examples. The sample complexity scales logarithmically with the number of candidate delimiters and inversely in a gap parameter between the prior likelihoods of correctly and incorrectly segmented sequences. 
\item \textbf{Inferring consistent hypothesis:} We then take the segmentation of the prompt and illustrate how to infer a consistent hypothesis which explains all the (input, output) pairs in the prompt using a transformer model. For this part, we specialize to a family of \emph{sparse retrieval tasks}, where the output is simply a token of the input, or the sum of a subset of input tokens. \emph{This family is a useful abstraction of token extraction and manipulation tasks in practical ICL settings.} We show how attention can naturally leverage correlations to identify a consistent hypothesis on such tasks. The proposed mechanism finds an $\epsilon$ accurate hypothesis from a class $\cF$ using $O\big(\frac{1}{\epsilon}\ln|\cF|\big)$ examples.
\item \textbf{Inference with the learned hypothesis:} We also show how the attention mechanism is well-suited to carry this hypothesis learned from the prompt and apply it to subsequent query inputs.
\item \textbf{Empirical validation:} Finally, we validate some of our theoretical findings through empirical validation, showing the dependence of ICL on easily identifiable delimiters. For hypothesis learning, we show that transformer models can be indeed trained to solve the sparse retrieval tasks studied here, and that the attention outputs correspond to the key steps identified in our theoretical mechanisms.
\end{itemize}

\section{Problem Setting and Notation}

A language model $\lm$ is an oracle that takes as input elements of a \textit{language} $\lang$, sequences of \textit{tokens} from a vocabulary $\vocab$, with $\vocabsize := |\vocab|$. A typical language model is autoregressive: it aims to predict the next sequence of tokens from an prefix. To complete the phrase ``I came, I saw'', we construct $\texttt{prompt = [\begintoken,I,\spacetoken,came,\commatoken,\spacetoken,I,\spacetoken,saw]}$, and input $\texttt{prompt} \to  \lm \to \texttt{output}$, and we expect that  $\texttt{output} = \texttt{[\commatoken,\spacetoken,I,\spacetoken,conquered,\endtoken]}$.

\subsection{The Transformer Architecture} \label{sec:transformers}

We describe the design of a language model using the architecture known as the decoder-only \textit{transformer} \citep{vaswani2017attention}. In short, transformers are models that process arbitrary-length token sequences by passing them through a sequence of layers in order to obtain a distribution over the next token.

For the following definition, we will need some special operators, which we describe here. The operation $\softmax(M)$ returns a matrix the same shape as $M$ whose $i,j$ entry is $\frac{\exp(M_{i,j})}{\sum_{j'} \exp(M_{i,j'})}$. The $\concat(M_1, M_2)$ operation stacks the matrices vertically.
The operation $\mask(M)$ takes a square matrix $M$ and returns $M'$ such that $M'_{i,j} = M_{i,j}$ for $i \leq j$ and $M'_{i,j} = -\infty$ otherwise. (The $-\infty$ is converted to a 0 after the $\softmax$ operation.) The \textsf{GeLu} operation is the Gaussian Error Linear Unit. 

\begin{definition}\label{def:transformerlayer}
Let $d, \datt, \kappa$ be arbitrary positive integers.
A \emph{transformer layer} is a function $\Transformer_\Pi$ parameterized by matrices $\Pi := \{Q_{k}, K_{k}, V_{k} \in \rR^{\datt \times d} \text{ for } k \in [\kappa], W_O \in \rR^{d \times \kappa \cdot\datt} \}$, that maps, for any length $N$ sequence, $\rR^{d \times N} \to \rR^{d \times N}$ using the following procedure:
\begin{align*}
    &\textbf{\upshape Input: }  X \in \rR^{d \times N},\quad
    \textbf{\upshape Set: } A_{k} \leftarrow \softmax \circ \mask (\datt^{-1/2}X^\top Q_k^{\top}K_{k}X)\quad \forall k \in [\kappa]\\
    &\textbf{\upshape Set: }  X' \leftarrow W_O\, \concat(V_{1} X A_{1}, \ldots, V_{\kappa} X A_{\kappa}),\quad 
    \textbf{\upshape Output: }  X + \textsf{\upshape GeLu}(X') \in \rR^{d \times N}.
\end{align*}
\end{definition}
We omit the layer normalization present in implementations~\citep{vaswani2017attention} for ease of presentation.
A convenient aspect of transformer layers is their composability. Assume we have $L$ transformer layers, where the $\ell$-th layer is parameterized by
$\Pi^{\ell} := \{ W_O^{\ell} \in \rR^{d \times \kappa\cdot\datt}; Q_k^{\ell}, K_k^\ell, V_k^\ell \in \rR^{\datt \times d}, k \in [\kappa] \}$.
The remaining piece we need for the full transformer model is the token embedding layer, which is parameterized by a matrix $W_E \in \rR^{d \times |\vocab|}$. If we take a prefix $x \in \lang$ with $N$ tokens, and write it using one-hot encoding as a matrix $Z \in \{0,1\}^{|\vocab| \times N}$, then $W_EZ$ is referred to as the ``embedded'' tokens. Once these embedded tokens are passed through one or more transformer layers to obtain $Z'$, we can convert back to vocab space by $W_E^\top Z'$.  
Here $Z_{i,j}$ represents the model's estimated probability that the $j+1^\text{th}$ token will be token $i$ given the first $j$ tokens in the sequence. We need these embeddings to be reasonably distinct.
\begin{definition}\label{def:transformer}
A (decoder-only) $L$-layer \emph{transformer} is a parameterized function that maps $\rR^{|\vocab| \times N} \to \rR^{|\vocab| \times N}$ for any sequence length $N$ where the input $X$ is a one-hot encoding of a token sequence $x$ in $\lang$, and the output is a column-stochastic matrix $Z$. The parameters are given by the matrix $W_E$ and the sequence $\Pi_{1}, \ldots, \Pi_{L}$. The full map is defined as the composition,
\[
X \mapsto Z = \textnormal{softmax} (W_E^\top \cdot \Transformer_{\Pi_{L}} \circ \cdots \circ \Transformer_{\Pi_{1}} (W_E \cdot X)).
\]
\end{definition}
The one-hot encoding can be replaced with other (possibly learned) encodings when $\vocab$ is large or infinite.
Of much interest in this work is to understand what operations can be implemented using a transformer. To establish our results, we often  show that certain operations $\rR^{d \times N} \stackrel{\phi}{\to} \rR^{d \times N}$ on embedded token sequences can be implemented using a transformer layer parameterized by $\Pi$. When there is a $\Pi$ such that $\Transformer_\Pi \equiv \phi$ for all $N$, then \emph{$\phi$ can be implemented as a transformer layer}.

\subsection{In-Context Learning}

Let us now imagine that we hope to solve the following learning problem. We are given an input space $\cX$ and output space $\cY$. We assume that $\cX, \cY \subset \lang$ for simplicity --i.e., we are able to express inputs/outputs in the given language. Assume we have a  set $\cF$ of functions $f : \cX \to \cY$. A \textit{task} in this setting is a pair $f, \dist$, with $f \in \cF$ a function and $\dist \in \Delta(\cX)$ a distribution on inputs $x \in \cX$. A \textit{sample} $S_n$ from this task is a collection of $n$ labelled examples $\{ (x_1, y_1), \ldots, (x_n, y_n) \} \subset \cX \times \cY$ where the $x_i$'s are samples i.i..d. from $\dist$ and $y_i = f(x_i)$ for every $i \in [n]$. When viewed as a typical supervised learning setting, we would design a \textit{learning algorithm} $\cA$ that is able to estimate $\hat f \in \cF$ from a sample $S_n$,
$\{ (x_1, y_1), \ldots, (x_n, y_n) \} \quad \to \quad \cA \quad \to \quad \hat f_n$.
The goal of $\cA$ is to minimize expected loss $\E_{x \sim \dist}[\text{loss}(f(x), \hat f_n(x))]$ with respect to a typical sample $x \sim \dist$ and (unknown) function $f$.

The in-context learning framework poses the idea that perhaps for a large family of tasks we do not need to design such an algorithm $\cA$ and instead we can leverage a pre-trained language model in order to solve a large family of learning tasks. That is, for a sample above and a test point $x \sim \dist$, we have the following setup
\[
\encoding(\{ (x_1, y_1), \ldots, (x_n, y_n) \}, x) \quad \to \quad \lm \quad \to \quad \texttt{output},
\]
and, if the ICL process succeeds, we expect that  $\texttt{output} = f(x)\endtoken$.

Since there is no canonical procedure to encode a set of example-label pairs guaranteed to be understood by $\lm$, the choice of \encoding influences its output behavior. Empirical works use typical \textit{delimiters}---special tokens that are typically used to give structure to documents by \emph{segmenting} text into lists, relations, etc.---for this task. As part of our language definition, we assume that there is a set of special tokens $\delims \subset \vocab$ including, e.g. punctuations (\commatoken, \colontoken, \semicolontoken), or spacing characters (\spacetoken, \newlinetoken, \tabtoken).  We assume that the user has selected one delimiter that separates the $n$ examples, which we will call \esep, and another  that distinguishes between $x_i$ and $y_i$, which we will call \lsep. The only requirement is that \esep and \lsep are distinct elements of $\delims$, and that \esep and \lsep do not occur in any $x,y$ examples generated in the task. With this in mind, we define $\encoding(\{ (x_1, y_1), \ldots, (x_n, y_n) \}, x)$ as
\begin{equation}
\begintoken  x_1  \lsep  y_1 \esep  x_2 \lsep y_2 \esep \ldots \esep x_n  \lsep  y_n \esep x \lsep.
\end{equation}
We note that the $x$'s and $y$'s have variable length and are being concatenated above.

\section{An Overview of Results}

We now survey the core results of the paper on segmenting the input sequence through delimiter identification, and the subsequent hypothesis learning.

\subsection{Segmenting an input sequence}
\label{sec:overview-parsing}

Suppose we have an underlying distribution $p_0(\cdot)$ on $\lang$, which measures the typical likelihood of sequences observed ``in the wild'', and this distribution is encoded in a transformer through the pre-training process. The goal of the segmentation mechanism is to identify a pair of separators $\lsep, \esep\in \delims\times\delims$, such that the input $z$ can be \emph{reasonably} decomposed as:
\[
z = \begintoken  x_1  \lsep  y_1 \esep \ldots \esep x_k  \lsep  y_k \esep x_{*} \lsep.
\]

To formalize a reasonable decomposition of $z$ obtained using delimiters $\lsep, \esep$, we define its likelihood by leveraging the base model $p_0$ and then follow a \emph{maximum likelihood segmentation}:
\begin{equation}
    \hlsep, \hesep = \argmax_{\substack{\lsep \in \delims\\\esep\in\delims}} p_0(x_*) \prod_{i=1}^k p_0(\begintoken x_i \endtoken)p_0(\begintoken y_i \endtoken),
    \label{eq:parsing-obj}
\end{equation}

We note that the number $k$ of examples as well as the sequences $x_i$ and $y_i$ identified depend on the separators used, and hence are all functions of the optimization variables $\lsep, \esep$ in the objective~\eqref{eq:parsing-obj}. This is a natural objective to decompose an input sequence, as it posits that the individual $x, y$ sequences in ICL should be plausible under the base distribution. Crucially, if the true label separator $\lsep^\star$ is very unlikely to occur in a natural sequence, then a wrong segmentation which mistakenly includes $\lsep^\star$ as part of some $x_i$ or $y_i$ will be very unlikely under $p_0$.

The first result of our paper is that the objective~\eqref{eq:parsing-obj} can be implemented using a transformer.

\begin{theorem}[Transformers can segment]
There exists a transformer with $O(1)$ layers and $O(\delims\times\delims)$ heads per layer which computes $\hlsep, 
\hesep$ according to~\eqref{eq:parsing-obj}.
\label{thm:parsing-mech-main}
\end{theorem}

Next, we evaluate the sample requirements to learn an accurate segmentation.

\begin{theorem}[Sample complexity of segmentation, informal]
Let $c$ measure how much more likely a correctly segmented sequence is than an incorrectly segmented onee under the task distribution $\dist$. Given a minimum probability parameter $\minprob$, maximum likelihood segmentation~\eqref{eq:parsing-obj} returns the correct label and example separators with probability $1-\delta$ after seeing $n = \Omega\left(\frac{(\log(1/\minprob))^2\log\frac{|\delims|}{\delta}}{c^2}\right)$.
\label{thm:parsing-sample-main-informal}
\end{theorem}

In practice, the example and label separators are chosen so as to make the segmentation fairly unambiguous, ensuring that $c$ is large, and the sample cost of learning segmentation is quite small. We can further enhance the objective~\eqref{eq:parsing-obj} to include priors over \lsep and \esep being delimiters to zoom in on typical choices faster. We omit this extension here to convey the basic ideas clearly.

\subsection{Learning a consistent hypothesis}
\label{sec:overview-learning}

Having generated a segmentation, the next step in ICL is to take the inputs $(x_i, y_i)_{i=1}^n$ identified above and generate a hypothesis $\widehat{f}$ such that $\widehat{f}(x) \approx f^\star(x)$, where $y_i = f^\star(x_i)$. To formalize the hypothesis learning setup, we focus on a specific family of learning problems that we define next.

\begin{definition}[Tokenized sparse regression]
\label{defn:gen-reg-task}
Fix an input space $\cX$, an output space $\cY$, a basis map $\basis(x)~:~\R^{\dim(\cX)}\to \R^m$, and distribution $\dist$ over $\cX$. Given $s \leq m$, an $s$-sparse tokenized regression problem is defined by weights $\beta_1,\ldots, \beta_m\in\{0,1\}^m$ such that $|\{j~:~|\beta_j| = 1\}| = s$ and $y = \sum_{j=1}^m \beta_j \basis(x)_j$. 
\end{definition}
In words, a tokenized sparse regression problem maps from inputs to outputs by taking sparse linear combinations of the inputs under some fixed basis transformation. We call the task tokenized due to the way in which the transformer processes the input $x$ during ICL, accepting it one coordinate at a time as we will see momentarily. This is to distinguish from the regression tasks studied in prior works~\citep{gargcan, akyurek2022learning, von2022transformers}, which consider each vector $x$ to be a token.

In particular, we study problems where the basis $\basis$ is fixed across contexts and only the coefficients $\beta_i$ vary across tasks. Since $\basis$ is fixed, it can be assumed to be known from pre-training and we focus on the case of $\cX = \R^m$ and $\basis(x)_i = x_i$, that is the basis is just the standard basis and the task is sparse linear regression. We focus on these tasks because \emph{extracting and manipulating a few input tokens seems emblematic of many of the string processing tasks where ICL is often used in practice}.

We analyze the following estimator for all $i \in [n]$:
\vspace{-0.1cm}
\begin{equation}
    f_i \in \big\{(j_1,\ldots,j_s)\in [m]^s~:~ \textstyle\max_{t=1,2,\ldots, i}|x_{t, j_1} + \ldots + x_{t,j_s} - y_t| \leq \epsilon\big\}.
    \label{eq:s-sparse-obj}
\end{equation}
\vspace{-0.1cm}
The optimization problem~\eqref{eq:s-sparse-obj} can be implemented with a transformer, as stated next.

\begin{theorem}[Transformers find a consistent hypothesis]
There exists a transformer with $O(m)$ layers and $1$ head per layer which computes an $f_i$ according to~\eqref{eq:s-sparse-obj} after reading example $x_i$.
\label{thm:learning-mech-main}
\end{theorem}

The estimator above is natural for the task, as it finds a solution with a zero loss on the training samples. Implementing this estimator is particularly natural with a transformer using properties of the attention mechanism, that is well suited to extracting coordinates of $x$ which are highly correlated with the label $y$. In fact, the actual mechanism computes a weighting over $\{1,2,\ldots,m\}$ as candidate solutions from each example, and then returns $f_i$ as the coordinate with the largest cumulative weight (across examples) after each example $i$. This provides further robustness in case of label noise. For this procedure, we provide the following sample complexity guarantee.

\begin{theorem}[Sample complexity of hypothesis learning, informal]
For any $\epsilon > 0$, suppose the initial token embeddings are such that tokens $z_\alpha, z_\gamma$ with $|z_\alpha - z_\gamma| \geq \epsilon$ have nearly orthogonal embeddings. Let $f_n$ be any hypothesis returned by Equation~\ref{eq:s-sparse-obj} after seeing $n$ examples from the $s$-sparse token regression task. Then for $n = \Omega(s\log(m/\epsilon)/\epsilon)$ we have $\rE[|f_n(x) - f^\star(x)|]\leq 2\epsilon$.
\label{thm:sample-complexity-s-sparse-main-informal}
\end{theorem}
The condition on the token embeddings is natural, since aliased tokens with different values can be problematic for learning. Our sample complexity matches the typical guarantees in sparse regression, which scale with the $\frac{s\log m}{\epsilon}$. Note that the estimator~\eqref{eq:s-sparse-obj} learns $f_n$ from scratch, in that there is no use of the pre-training to bias the estimator towards certain functions. While this is also done in several prior works on ICL~\citep{gargcan, akyurek2022learning, von2022transformers}, in practice ICL is often used in settings where the correct $y$ has a high probability under the base distribution $p_0$, given $x$. Improving the estimator~\eqref{eq:s-sparse-obj} to prioritize among the consistent tokens $j$ using the prior probabilities $p_0(x_{i,m+1} | x_{i,1},\ldots,x_{i,m})$ is an easy modification to our construction and allows us to further benefit from an alignment between $p_0$ and $\cD$.

\section{Segmenting an ICL Instance}
\label{sec:segment}

In this section, we provide more details about the segmentation mechanism and the underlying assumptions. Given a sequence of $N$ tokens $z = z_1,\ldots,z_N$, and given a pair $\sigma = (\lsep, \esep)$, we first segment the sequence into as many chunks as possible, separated by \esep tokens:
\[
z = \begintoken  \chunk_1 \esep \chunk_2 \esep \ldots \esep \chunk_k.
\]
Then, for each chunk we segment $\chunk_i = x_i \lsep y_i$ according to the occurrence of \lsep. If we find multiple occurrences of \lsep, we infer that this $\sigma$ is infeasible, setting $p_\sigma(z) = \minprob$ to be a minimum probability. If no \lsep is found, then we set $\chunk_i = x_i$. With this segmentation in mind, we can now define our probability model $p_\sigma$ as follows,
\begin{equation}
\label{eq:sigma_segmentation_distr}
    \textstyle p_\sigma(z) = p_0(x_*) \prod_{i=1}^k p_0(\begintoken x_i \endtoken) p_0(\begintoken y_i \endtoken).
\end{equation}

\begin{figure}[t!]
    \centering
{\footnotesize
\begin{verbatim}
Double, double toil and trouble: / Macbeth; To die, - To sleep, - To sleep! 
/ Hamlet; This above all: to thine own self be true / Hamlet; A deed without a name /
\end{verbatim}
}
\caption{\begin{small}An example of an ICL task: quotes from Shakespeare followed by the name of the play. Correct delimiters are $(\lsep, \esep) = (\texttt{/},\texttt{;})$, yet the presence of other potential delimiters creates ambiguity.\end{small}}
    \label{fig:segmentation}
    \vspace*{-10pt}
\end{figure}

To understand the meaning of this definition better, it is helpful to look at an example. In Figure~\ref{fig:segmentation}, we show a concatenation of quotes from Shakespeare and the corresponding play names. Using candidate delimiter pairs $\sigma = (\texttt{/},\texttt{;})$ and $\sigma' = (\texttt{:},\texttt{-})$, we get two different likelihood models:
\begin{align*}
    p_\sigma(z) &= p_0(\begintoken\texttt{Double,...trouble:}\endtoken)p_0(\begintoken \texttt{Macbeth}\endtoken)\\
    &\qquad\cdot p_0(\begintoken \texttt{To die, }\endtoken)...\\
    p_{\sigma'}(z) &= p_0(\begintoken \texttt{Double,...trouble}\endtoken)p_0(\begintoken \texttt{/ Macbeth;...To die,}\endtoken)\\&\qquad\cdot p_0(\begintoken \texttt{To sleep,}\endtoken)p_0(\begintoken\endtoken)...
\end{align*}
It is also important to note that, even though the model above provides a nicely-factored estimate of the sequence likelihood, we are still able to compute the probability of other segments. For any $1 \leq i < j \leq N$, $p_\sigma(z_i\cdots z_j | z_1\cdots z_{i-1})$ can be evaluated from the model above, even if $[i:j]$ may cross delimiter boundaries. The objective~\eqref{eq:parsing-obj} maximizes $p_\sigma(z)$ over $\sigma$ to select the model.

\paragraph{Implementation using a transformer.} We now describe a high-level sketch of a transformer that can implement the objective~\eqref{eq:parsing-obj}. As mentioned in Theorem~\ref{thm:parsing-mech-main}, our construction uses one head for each candidate $\sigma$, where we use the head corresponding to $\sigma$ to compute $p_\sigma(z_{1:i})$ at each token $i$. The maximum over delimiter pairs $\sigma$ by looking across heads is subsequently taken by the output MLP layers. Here we focus on the operations within each head. For a fixed $\sigma = (\lsep, \esep)$, the first two layers of the transformer identify the nearest occurrences of \esep and \lsep to each token $i$. This can be done by attending to the occurrence of these tokens with the largest token, which is natural with softmax and position embeddings. Next, we find the first occurrence of \lsep following each \esep and add the log probabilities of the subsequence from \esep to the token before \lsep and from the token after \lsep to $z_i$. Implementing these operations is again relatively straightforward using a composition of soft attention layers, assuming access to a certain conditional probability evaluation module from pre-training, that returns log of conditional probability of a token, conditioned on a prefix sequence. 
Adding all the log probabilities gives us the desired output. Details of this construction are provided in Appendix~\ref{app:parsing}.

\paragraph{Sample complexity of segmentation.} Let $\sigma^\star$ be the true pair of delimiters used to generate the input $z$. Then we expect the procedure~\eqref{eq:parsing-obj} to return $\sigma^\star$ only if the sequences $\chunk_i$ identified under $\sigma^\star$ are more probable than those identified under a different delimiter $\sigma'$. To facilitate this, we now make an assumption, before stating the formal sample complexity guarantee.

\begin{assumption}\label{assm:parser}
Let $\dist \in \Delta(\lang)$ and $f$ define the task at hand, with the input segmentation determined by $\sigma^\star = (\lsep, \esep)$. For an ICL sequence $z$, let $\chunk_i = x_i \lsep y_i$ be the $i$th example/label chunk (under $\sigma^\star$). Consider the expected log-likelihood ratio of $\chunk_i$ conditioned on its prefix according to two probability models, one with the correct segmentation using $\sigma^\star$ and alternatively with some other delimiter pair $\sigma'$, conditioned on $\mathsf{prefix}_{i-1}$, which is the prefix of the sequence before $\chunk_i$.
Now, for some $c > 0$ we assume that for any $\sigma' \ne \sigma^\star$ and with probability $1$ for any $\mathsf{prefix}_{i-1}$
\[
    \E_{\substack{x_i \sim \dist\\y_i=f(x_i)}}\left[ \log \frac
        {p_{\sigma^\star}(\chunk_i | \mathsf{prefix}_{i-1})}
        {p_{\sigma'}(\chunk_i | \mathsf{prefix}_{i-1})} 
        \; \vert \; \mathsf{prefix}_{i-1} \right]
    \geq c.
\]
\end{assumption}
Note that in both cases above, the segmented chunks, $\chunk_i$, are always determined by $\sigma^\star$ even though their likelihood is evaluated on the ``false'' $\sigma'$.
What does this assumption mean? It may appear to be highly technical, but it encodes something that we would very naturally expect: incorrectly segmented data should look very weird (unlikely) relative to correctly interpreted data. For instance, revisiting the example from Figure~\ref{fig:segmentation}, consider the second chunk according to the true segmentation with $\sigma^\star = (\texttt{/},\texttt{;})$, $u_i=$``To die, - To sleep, - To sleep! / Hamlet;''. When correctly segmented we obtain the pair $x_i = \texttt{To die, - To sleep, - To sleep!}$ and $y_i = \texttt{Hamlet}$, and the likelihood $p_{\sigma^\star}(u_i) = p_0(x_i)p_0(y_i)$. On the other hand, when we use the incorrect delimiters $\sigma' = (\texttt{:},\texttt{-})$ we get a much less natural segmentation, with model estimate
\begin{align*}
    p_{\sigma'}(u_i \; \vert \; \mathsf{prefix}_{i-1}) = & p_0(\texttt{To die,}\endtoken \; \vert \; \texttt{\begintoken / Macbeth;}) \cdot p_0(\begintoken\texttt{To sleep,\endtoken})\\
    & \quad \cdot  p_0(\begintoken\endtoken) p_0(\begintoken\texttt{To sleep! / Hamlet}).
\end{align*}
Assumption~\ref{assm:parser} says that the model estimate for these chunks according to $p_{\sigma^\star}$ should be much higher, on average, than that for $p_{\sigma'}$. This is indeed the only needed assumption to obtain the following.

\begin{theorem}
Let $\minprob>0$ be such that $p_{\sigma'}(\chunk_i | \mathsf{prefix}_{i-1}) \geq \minprob$ where $\chunk_i$ is the $i$th chunk under $\sigma^\star$, for all $\chunk_i$ and $\mathsf{prefix}_{i-1}$ almost surely.
Under Assumption~\ref{assm:parser}, the maximum likelihood segmentation algorithm \eqref{eq:parsing-obj} outputs the correct delimiters $\sigma^\star$ w.p. at least $1-\delta$, as long as 
    $n \geq \frac{16\left(\log \frac 1 \minprob\right)^2 \log \frac {|\delims|}{\delta} }
        {c^2}$.
\label{thm:parsing-sample-main}
\end{theorem}

\section{Learning a Consistent Hypothesis for Tokenized Sparse Regression}
\label{sec:hyp-learning}
We now formalize the results for extracting a consistent hypothesis for the tokenized sparse regression tasks described in Definition~\ref{defn:gen-reg-task}. For intuition, we begin with $s=1$.

\textbf{The $1$-sparse tokenized regression task:} Recall that when $s = 1$, we have $y_i = x_{i, f^\star}$ for each example $i\in[n]$, and $f^\star\in[m]$ is the coordinate of $x_i$ being copied to $y$. The objective~\eqref{eq:s-sparse-obj} similarly simplifies to finding an index $f_i \in [m]$ such that $|x_{t,j} - y_t| \leq \epsilon$ for all examples $t \leq i$. We now describe the key elements of a transformer that implements such a procedure; details in Appendix~\ref{sec:token_task}. The construction is depicted in Figure~\ref{fig:icl_example}. Let $z_i^1$ be the initial token embedding, i.e. the $i$th column of $W_E X$ in Definition~\ref{def:transformer}.

\begin{assumption}
\label{assm:orth}
Let $z_\alpha^1$ be the input embedding of a token $z_\alpha$. For any $z_\alpha, z_\gamma$, such that $|z_\alpha - z_\gamma| \geq \epsilon$, we have $|\langle z_\alpha^1, z_{\gamma}^1 \rangle| < \frac1{2\tau}$, for some $\tau \geq 1$. If $|z_\alpha - z_\gamma| \leq \epsilon$, then $\langle z_\alpha^1, z_\gamma^1 \rangle \geq 0$ and $\langle z_\alpha^1, z_\alpha^1 \rangle = 1$.
\end{assumption}

One embedding which satisfies the assumption sends all $z_\alpha, z_\gamma$ s.t. $|z_\alpha - z_\gamma| \geq \epsilon$ to orthogonal vectors in $\R^{\lceil(1/\epsilon)\rceil}$. In Appendix~\ref{sec:token_task} we discuss alternatives to cut the dimension from $O(\frac1\epsilon)$ to $\tilde O(\tau^2)$.

\begin{figure}[t!]
    \begin{minipage}{0.4\textwidth}
    \includegraphics[width=0.55\textwidth]{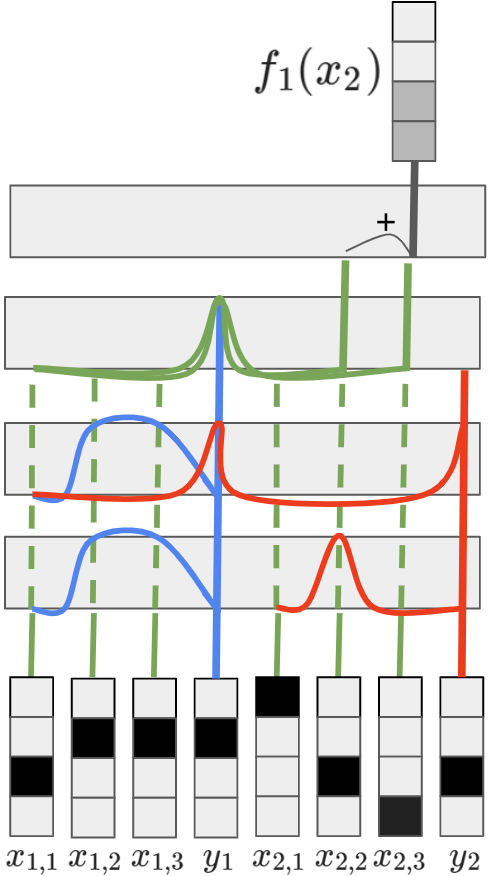}
    \end{minipage}
    \hspace*{-50pt}
    \begin{minipage}{0.72\textwidth}
    \caption{\begin{small}A transformer for $1$-sparse tokenized regression with $n=2$ examples and $m=3$ tokens per example. The curved lines show attentions, with heights proportional to the attention. The blue and red attention lines show the attentions of $y_1$ and $y_2$ over the previous tokens. The green attention lines show the attentions of $x_{2,2}$ and $x_{2,3}$ over the previous tokens. In this case, $f^\star=2$. After the first example, there is ambiguity between $f_1 \in\{2,3\}$, hence the output $f_1(x_2)$ mixes theseand is not correct. After the second example, the answer is uniquely determined, for inference on third example and beyond. In the first layer, each $y_i$ attends to tokens $x_{i,j}$ from example $i$ to find all consistent hypotheses in example $i$. By attending across previous $y_t$'s, each $y_i$ aggregates these hypotheses over all preceding inputs $t \leq i$. Example $i+1$ then attends to $y_i$ to predict using the aggregated hypothesis in the final two layers.\end{small}}\label{fig:icl_example}
    \end{minipage}
    \vspace*{-15pt}
\end{figure}

Given such an embedding, we can now use inner products between tokens to detect similarity, which is the key first step in our construction. This is also very natural to implement using attention. We define the first attention layer so that each $x_{i,j}$ only attends to itself and $y_i$ only attends to $x_{i,1},\ldots,x_{i,m}$ using the position embeddings. The attention weight between $y_i$ and $x_{i,j}$ is defined using the inner product of their first layer embeddings. By Assumption~\ref{assm:orth}, this inner product is large for $j = f^\star$ and small whenever $|x_{i,j} - y_i| > \epsilon$. In Appendix~\ref{sec:token_task}, we define the query and key matrices to induce such an attention map. Under Assumption~\ref{assm:orth}, this attention map identifies a good hypothesis for example $i$, as formalized below.

\begin{lemma}
Given an example $i$, let $\cJ_i = \{j~:~|x_{i,j} - y_i| \leq \epsilon\}$, and let $f^\star \in [m]$ be such that $y_i = x_{i, f^\star}$. Under Assumption~\ref{assm:orth}, for any $m \geq 2$, the output of the first layer at $f^\star$ is larger than the output at any $j\in [m]\setminus \cJ_i$ by at least $e/(4(m+1))$.
\label{lemma:margin}
\end{lemma}
This key step in our construction gives us some consistent hypothesis from $\cJ_i$, for each example $i$ \emph{individually}. The second layer now finds a hypothesis consistent with all examples $(x_t, y_t), t\leq i$, and saves this hypothesis in token $y_i^3$, which is the input to the third layer. 
The remaining two layers implement appropriate copy and extraction mechanisms for the hypothesis $f_i$ identified at $y_i$ to be applied to the next input $x_{i+1}$ by first extracting the value at $x_{i+1, f_i}$ and then outputting this value at the token $x_{i+1, m}$ as the ICL prediction at the $(i+1)$th example. We illustrate the key points of the construction in Figure~\ref{fig:icl_example}. For more details, we refer the reader to Appendix~\ref{sec:token_task}.

Prior work has focused on tasks where $x_i$ is not decomposed into one token per coordinate but rather given to the transformer as a vector in $\R^m$ directly. In Appendix~\ref{app:vector_task} we demonstrate how this setting can be reduced to the tokenized setting that we study here.

\textbf{Iterative deflation for $s$-sparse token regression:} For the more general case, it is tempting to directly apply the $1$-sparse construction once more and hope that all the tokens in $f^\star$ will be identified simultaneously, since they all should have a high inner product with $y_i$ under Assumption~\ref{assm:orth}.
However, suppose that $f^\star = (1, 2)$ and let us say that the hypothesis $(1, 3)$ is also consistent. Then an approach to learn 2 coordinates independently using the approach from the previous section can result in the estimate $(2, 3)$ which might not be consistent. To avoid this, we identify coordinates one at a time, and deflate the target successively to remove the identified coordinates from further consideration.

The deflation procedure works because of the following crucial lemma.
\begin{lemma}
\label{lem:sep}
Under Assumption~\ref{assm:orth} with $\tau \geq 2s$, for any $\mathcal{C} \subseteq f^\star$ and, $j \in f^\star\setminus\mathcal{C}$, we have $\langle x_{i,j}^1, y_i^1 - \sum_{j'\in\mathcal{C}} x_{i,j'}^1\rangle \geq \frac{3}{4}$, while if $|x_{i,j} - x_{i,j'}| \geq \epsilon$ for all $j' \in f^\star$ then we have $\langle x_{i,j}^1, y_i^1 - \sum_{j'\in\mathcal{C}} x_{i,j'}^1\rangle \leq \frac{1}{4}$.
\end{lemma}
This lemma says that the deflated target favors unidentified coordinates in $f^\star$.
Next, we show that the deflation procedure, to remove all previously identified tokens from each $y_i$, can be implemented using attention. We then stack $O(s)$ layers of the $1$-sparse task, with a deflation layer after extracting each coordinate to identify a complete consistent hypothesis that optimizes the objective~\ref{eq:s-sparse-obj}. 

Putting everything together, we have the following formal version of the earlier informal Theorem~\ref{thm:sample-complexity-s-sparse-main-informal}

\begin{theorem}
For any $\epsilon > 0$, let $f_n$ be some optimum of~\ref{eq:s-sparse-obj} after seeing $n$ examples from the $s$-sparse token regression task. Then there exists an embedding of $x_{i,j}, y_i,\forall i\in[n], j\in[m]$ in $\R^{O( s/\epsilon)}$ satisfying Assumption~\ref{assm:orth}, such that for any $n = \Omega(s\log(m/\epsilon)/\epsilon)$, w.p. $1-\delta$, $\E[|f_n(x) - f^\star(x)|] \leq 2\epsilon$.
\label{thm:sample-complexity-s-sparse-main}
\end{theorem}
More details for the construction and the proof of the sample complexity are in in Appendix~\ref{app:s-sparse}.

\section{Empirical Results}
\label{sec:experiments_main}
In this section, we report some empirical findings on the $1$-sparse tokenized regression task. For other experiments detailing the sensitivity of ICL to delimiters, we refer the readers to Appendix~\ref{app:segmentation_experiments}. 

We follow the experimental setup from \citet{gargcan} and \citet{akyurek2022learning}, building on the implementation of \citet{akyurek2022learning}. We use transformers with $8$ layers, 1 head per layer and embedding size $128$. We experiment on the $1$-sparse tokenized regression task with $(x_i, y_i)_{i\in [n]}$ generated in the following way. First a random hypothesis $f^\star$ is drawn uniformly at random from $[m]$, where $m=5$. Next, $x_i\in \R^m$ is sampled i.i.d from a fixed distribution which is either a standard Gaussian ($x_i\sim \mathcal{N}(0, I_{5\times 5})$), or uniform over $\{+1, -1\}^5$ ($x_i \sim Unif(\{+1, -1\}^5)$). $y_i$ is always set to $f^\star(x_i)$. We refer to the first setting as the Gaussian setting and the second as the Rademacher setting. We train three different transformers, one for each of the two settings, and one where the samples come from a uniform mixture over both Gaussian and Rademacher settings. After pre-training we carry out the ICL experiments by generating $64$ example sequences of length $5$, either all from the Gaussian setting or the Rademacher setting. A randomly drawn $f^\star$ is sampled and fixed and shared across these $64$ sequences, to allow averaging of results across multiple sequences. 

In Figure~\ref{fig:mh} we show the averaged loss of a model pre-trained on the uniform mixture of the settings, attention weights at the final layer (8) and attention weights at layer 6 for both settings. Appendix~\ref{app:token_task_experiments} shows results for models trained using Gaussian or Rademacher examples only. For the loss plots, $y$-axis is the loss of the ICL inference at each example and $x$-axis is the number of in-context examples observed. \emph{All examples are indexed from $0$}. We see that the model reaches a loss of $0$ in the Gaussian setting from a single sample, which is information theoretically optimal. In the Rademacher setting there are often multiple coordinates consistent with $f^\star$ on the first $3$ or $4$ examples, which results in higher loss compared to the Gaussian setting. The model typically learns $f^\star$ as soon as the correct coordinate is disambiguated in the Rademacher setting as well. 

\begin{figure}
\centering
\begin{subfigure}[t]{.32\textwidth}
    \includegraphics[width=\textwidth]{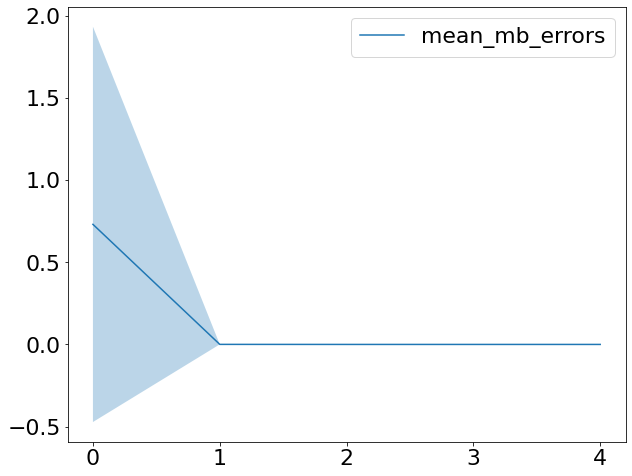}
    \caption{\begin{small}Loss (Gaussian), $y_i = x_{i,0}$\end{small}}
\end{subfigure}
\begin{subfigure}[t]{0.33\textwidth}
    \includegraphics[width=\textwidth]{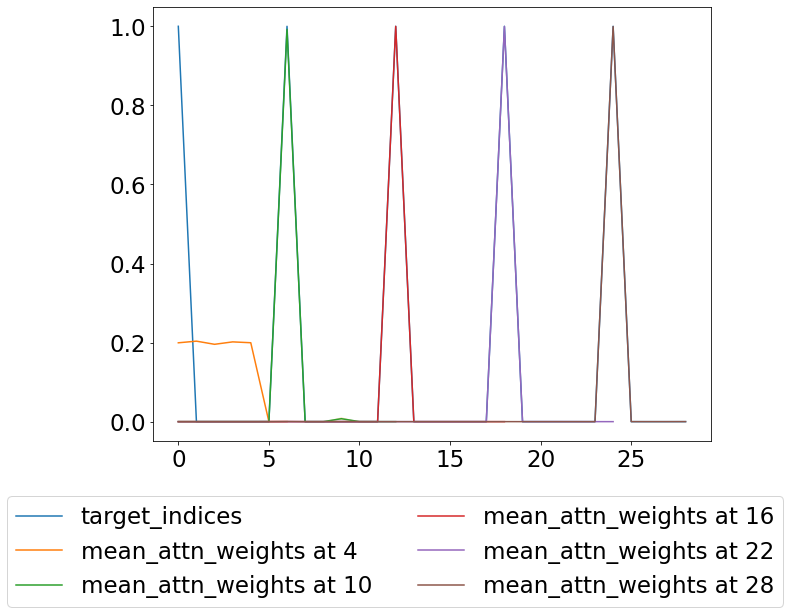}
    \caption{Attention at layer 8 (Gaussian)}
\end{subfigure}
\begin{subfigure}[t]{0.33\textwidth}
    \includegraphics[width=\textwidth]{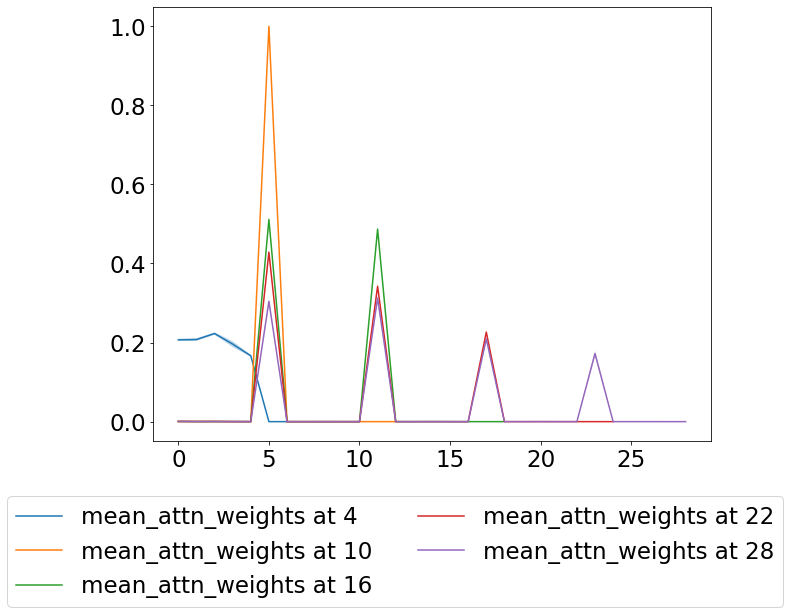}
    \caption{Attention at layer 6 (Gaussian)}
\end{subfigure}
\begin{subfigure}[t]{.32\textwidth}
    \includegraphics[width=\textwidth]{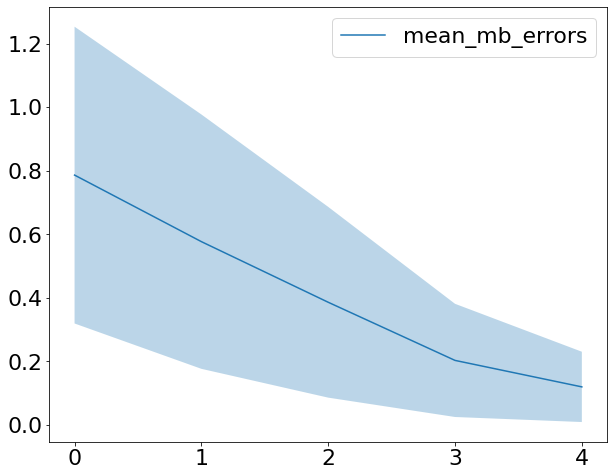}
    \caption{\begin{small}Loss (Rademacher). $y_i = x_{i,2}$\end{small}}
\end{subfigure}
\begin{subfigure}[t]{0.33\textwidth}
    \includegraphics[width=\textwidth]{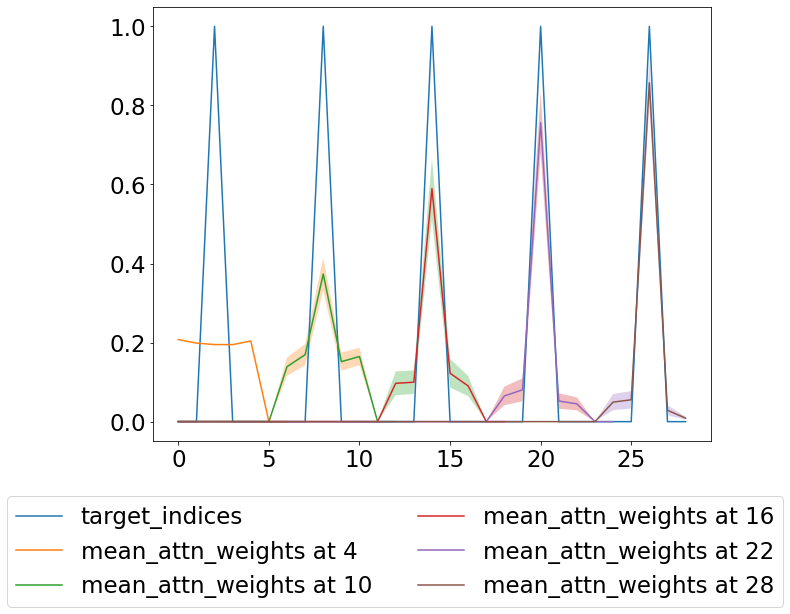}
    \caption{Attention at 8 (Rademacher)}
\end{subfigure}
\begin{subfigure}[t]{0.33\textwidth}
    \includegraphics[width=\textwidth]{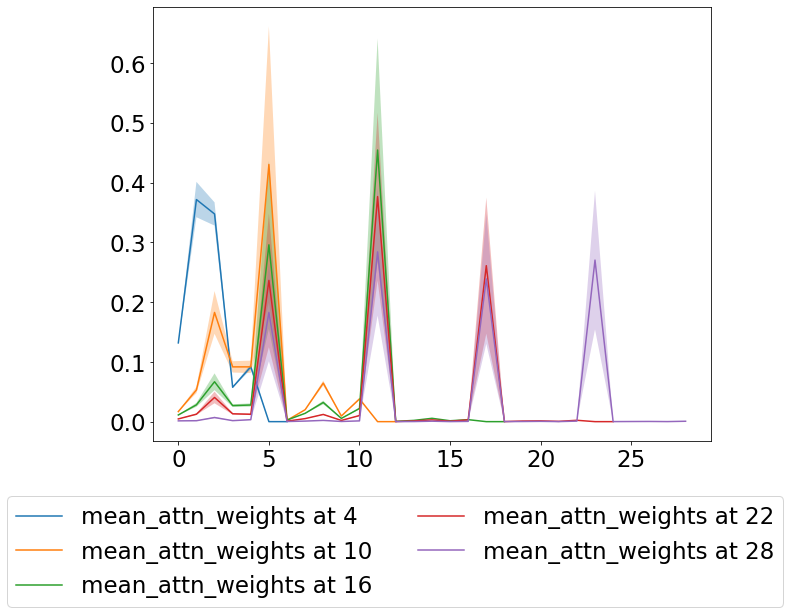}
    \caption{Attention at 6 (Rademacher)}
\end{subfigure}
    \caption{\begin{small}Loss and attention plots for $1$-sparse tokenized regression for Gaussian (top) and Rademacher (bottom) inputs. Loss drops to zero as soon as $f^\star$ is determined, and attentions follow the construction of Section~\ref{sec:hyp-learning}. Indices $4, 10, 16,\ldots$ are tokens where the label is predicted. In panels (b) and (e), these indices attend to the index of $f^\star$ in $x_i$ to predict $y_i$ correctly. The target indices line (blue) in panel (b) perfectly overlaps with the attention spikes at tokens $x_{i,0}$. In panel (d), the attention spikes largely overlap with target indices, but there is some noise (see text). In panels (c) and (f), these indices attend to all previous labels (indices $5, 11, 17,\ldots$) to aggregate a consistent hypothesis across previous examples. \end{small}}
    \vspace*{-12pt}
    \label{fig:mh}
\end{figure}

In Figures~\ref{fig:mh} (b) and (e) we plot the attention of the tokens at which the model outputs its predicted label, corresponding to the last $x$ token. For example $i$ (starting at $i=0$), this token is at index $6*(i+1)-2$ in our sequence, and we plot the attention weights from this token over all previous tokens in the sequence. We also show (in blue) the index corresponding to $y$, identified by $f^\star$. For the Gaussian setting (b), we see that the attention weights, starting at example $i=1$, are peaked at $f^\star$ (so these lines perfectly overlap with the target indices line in the plot). For instance, $y_i = x_{i,0}$ here and token $10$ attends to token $6$ accordingly to correctly predict $y_1$. In the Rademacher setting (e) (where $y_i = x_{i,2}$), we see that there is more variance, due to the fact that there are often multiple consistent hypothesis for the first few examples, however as the transformer processes more examples from the sequence the attention becomes more peaked at $f^\star$. This behavior is consistent with our theoretical construction. We note that our theoretical construction from Section~\ref{sec:hyp-learning} implies that attention weights in the last layer should be split uniformly over all $j$ which are consistent with $f^\star$ up to example $i$. In Appendix~\ref{app:experiments} we also empirically demonstrate that this is the case by looking at the attention on a single randomly sampled sequence.

Finally in Figure~\ref{fig:mh} (c) and (f) we plot the attention of the same tokens but at layer $6$ of the transformer, that is the third layer counting from the final layer. We see that attentions are peaked at the tokens holding the labels for $f^\star$, that is $x_{i, 5}$. This is analogous to the step where $y_i$ attends to all previous $y_s$ in our construction to aggregate across examples, and we expect its role is the same here.

In summary, we find that the attention maps in these experiments bear striking correspondence to our theory. We refer the reader to Appendix~\ref{sec:token_task} for more results that validate this correspondence.

\section{Conclusion}

In this paper, we take a fresh look at the in-context learning capability of transformers. We provide mechanisms that can implement sequence segmentation and hypothesis learning for a family of ICL tasks, and provide statistical guarantees showing that a fairly small number of examples indeed suffice to convey a target concept using ICL. More broadly, the ability of ICL to demonstrate information theoretically optimal learning in the types of tasks used both here and in prior works~\citep{gargcan, von2022transformers, akyurek2022learning} is quite impressive. It would be interesting to understand if there are learning tasks where this optimality fails to hold in ICL, and determine the necessary scaling of model size as a function of problem size needed to achieve sample-optimal learning, when possible. 
 
\bibliographystyle{plainnat}
\bibliography{refs}

\clearpage

\appendix
\input{appendix}


\end{document}

%% file: myheaders.tex
\usepackage{subcaption}
\usepackage{graphicx}

\usepackage{amsmath,amsfonts,amssymb, amsthm}
\usepackage{mathrsfs}
\usepackage{xcolor}
\usepackage{algorithm, algorithmic}
\usepackage{xspace}
\usepackage{natbib}
\usepackage{bm}
\usepackage{enumitem}
\newtheorem{lemma}{Lemma}
\newtheorem{theorem}{Theorem}
\newtheorem{assumption}{Assumption}
\newtheorem{definition}{Definition}

\theoremstyle{definition}

\newcommand{\cA}{{\mathcal A}}
\newcommand{\cX}{\ensuremath{\mathcal{X}}}
\newcommand{\cY}{\ensuremath{\mathcal{Y}}}
\newcommand{\cD}{{\mathcal D}}
\newcommand{\cF}{\ensuremath{\mathcal{F}}}

\newcommand{\R}{\ensuremath{\mathbb{R}}}

\renewcommand{\P}{\ensuremath{\mathbb{P}}}

\newcommand{\cV}{\ensuremath{\mathcal{V}}}

\def\vocab{\mathcal{V}}
\def\lang{\vocab^{*}}
\def\vocabsize{V}
\def\dist{\mathcal{D}}
\def\lm{\mathbb{LM}}

\DeclareTextFontCommand{\texttt}{\ttfamily}

\def\begintoken{\texttt{{\upshape <begin>}}\xspace}
\def\endtoken{\texttt{{\upshape <end>}}\xspace}
\def\lsep{\texttt{{\upshape <lsep>}}\xspace}
\def\esep{\texttt{{\upshape <esep>}}\xspace}
\def\commatoken{\texttt{{\upshape <comma>}}\xspace}
\def\colontoken{\texttt{{\upshape <colon>}}}
\def\semicolontoken{\texttt{{\upshape <semicolon>}}\xspace}
\def\newlinetoken{\texttt{{\upshape <newline>}}\xspace}
\def\tabtoken{\texttt{{\upshape <tabspace>}}\xspace}

\def\spacetoken{\texttt{{\upshape <space>}}\xspace}
\def\delims{\vocab_{\texttt{\upshape delims}}\xspace}
\def\encoding{\texttt{Encoding}\xspace}

\def\dout{d_{\text{out}}}
\def\datt{d_{\text{att}}}
\def\Transformer{\mathbb{T}}

\def\softmax{\mathsf{softmax}}
\def\mask{\mathsf{mask}}
\def\concat{\mathsf{concat}}
\def\gelu{\mathsf{GeLU}}

\DeclareMathOperator*{\rE}{{\mathbb E}}
\DeclareMathOperator{\E}{{\mathbb E}}

\DeclareMathOperator*{\argmax}{\text{argmax}}

\newcommand{\inner}[2]{\ensuremath{\left\langle #1, #2 \right\rangle}}

\usepackage{eqparbox}

\newcommand{\rR}{{\mathbb R}}

\newcommand{\ind}{\ensuremath{\bm 1}}

\renewcommand{\dim}{\ensuremath{\mathrm{dim}}}

\newcommand{\din}{\ensuremath{d_{\text{in}}}}
\newcommand{\cJ}{\ensuremath{\mathcal{J}}}
\newcommand{\cC}{\ensuremath{\mathcal{C}}}
\newcommand{\basis}{\psi}

\newcommand{\delim}{\mathfrak{\delta}}
\def\hlsep{\widehat{\text{\texttt{\upshape <lsep>}}}\xspace}
\def\hesep{\widehat{\text{\texttt{\upshape <esep>}}}\xspace}

\newcommand{\minprob}{\ensuremath{\nu}}
\newcommand{\chunk}{u}

%% file: appendix.tex
\section{Related Work}
\label{sec:related}

There is a rich literature on fine-tuning a pre-trained model through various forms of meta-learning or bi-level optimization~\citep{schmidhuber1996simple, andrychowicz2016learning, finn2017model}, which typically update all or some of the model's parameters to adapt it to a target task, with theoretical underpinnings either through direct analysis~\citep{finn2019online} or through the broader literature on transfer learning~(see e.g.~\citep{redko2020survey} and rerefences therein). Starting with the work of~\citet{brown2020language}, and with careful prompt design, this has become a predominant mechanism for adaptation in large language models~\citep{zhao2021calibrate, gao2020making, wei2022chain}, including to complex scenarios such as learning reinforcement learning algorithms~\citep{laskin2022context}. A consistently striking feature of these works is that it takes just a handful of examples from the target task to teach a new concept to the language model, once the inputs have been \emph{carefully phrased and formatted}.

This has naturally motivated a number of studies on the mechanisms underlying ICL, as well as its sample complexity. A formal treatment of this topic is pioneered in~\citet{gargcan}, who study ICL for linear regression problems, by pre-training a transformer architecture specifically for this task.
That is, the model is trained on prompts consisting of $(x,y)$ pairs from a linear regression instance, with the regression weights randomly chosen for each prompt. They showed that ICL can learn this task using a prompt of $O(d)$ examples when $x \in \R^d$. Subsequent work of~\citet{von2022transformers} proposes an explanation for ICL in this task by showing that self-attention can simulate gradient descent steps for linear regression, so that the model can effectively learn an optimizer during pre-training. \citet{akyurek2022learning} further extend this by suggesting that given enough parameters, the model can compute a full linear regression solution. Both works present some empirical evidence as well suggesting that these operations correspond to the final outputs and some intermediate statistics of the transformer, when trained for this task. \citet{dai2022can} study the relationship between linear attention and gradient descent, and \citet{li2023transformers} study transformers as producing general purpose learning algorithms. \citet{xie2021explanation} and \citet{zhang2022analysis} cast ICL as posterior inference, with the former studying a mixture of HMM models and the latter analyzing more general exchangeable sequences. \citet{wies2023learnability} give PAC guarantees for the sample complexity of ICL, when pre-trained on a mixture of downstream tasks. \citet{olsson2022context} and \citet{elhage2021mathematical} view ICL as an algorithm which copies concepts previously seen in a context example and then does inference by recalling these concepts when a new prompt matching previous examples occurs. \citet{elhage2021mathematical} explain this behavior formally for transformers with a single attention head and two layers and \citet{olsson2022context} conduct an empirical study on a wider variety of tasks for larger transformers.

Despite this growing literature, many aspects of the ICL capability remain unexplained so far. First, only \citet{li2023transformers}, \citet{wies2023learnability} and \citet{zhang2022analysis} provide any kind of sample complexity guarantees. Of these, the pre-training distribution in \citet{wies2023learnability} is too specific as the downstream task mixture, while \citet{li2023transformers} depend on an measure of algorithmic stability that is hard to quantify apriori. Secondly, all the works with the exception of \citet{xie2021explanation} require that the prompt has already been properly parsed into input and output examples, so as to facilitate the explanation of learning in terms of familiar algorithms, and the explanation of \citet{xie2021explanation} relies on a particular mixture of HMMs model. Further, we note that none of these works take into consideration the specifics of the transformer architecture and how self-attention can implement the proposed learning mechanisms.

While we do not study the properties of the pre-training process and data distribution in the ICL capability, these factors have been found to be crucial in empirical investigations~\citep{chan2022data, shin2022effect}, and expanding the theoretical model to address such phenomena is an important direction for future research.

\section{Notation}

We denote the input sequence as $z \in \vocab^N$ consisting of $N$ tokens. Each sequence is assumed to contain up to $n$ i.i.d. samples corresponding to a task, where sample $i$ consists of an example $x_i \in \vocab^{m'}$ and label $y_i \in \vocab$, where $m' \leq m$ is the token length of example $x_i$. 
For majority of the learning tasks that we analyze, the label $y_i$ consists of only 1 token, and $x_i$ consists of up to $m$ tokens. We use $x_{i,j} \in \vocab$ for $j\in[m]$ to denote the tokens of $x_i$.
The ICL instance is encoded as we have already described, where each $x_i$ and $y_i$ is separated by $\lsep \in \delims$, and $\esep \in \delims$ is found between between each successive pair $(x_i,y_i)$, where $\sigma = (\lsep, \esep)$ are  segmentation tokens. We will often refer to the ``ground truth'' segmentation token pair as $\sigma^\star$, where $\sigma \in \delims \times \delims$ is an arbitrary pair of delimiters. 
We also recall the notation $p_{\sigma^\star}(z) = \Pi_{i=1}^n p_0(x_i)p_0(y_i)$ for some underlying distribution $p_0$ on $\cV^*$, where the segmentation is determined by $\sigma^*$, and we remind the reader that we can analogously define $p_\sigma(z)$ for any other delimiter, which segments $z$ into another sequence of $(x,y)$ pairs.

When it comes to constructing transformer architectures, we will generally follow the same notation laid out in Section~\ref{sec:transformers}, but with a few new symbols and additional conventions. When we write $\inner{u}{v}_{M}$ for matrix $M$ and vectors $u,v$ we mean the inner product $\inner{uM}{v} = \inner{u}{Mv}$. For characters $z,x,y,o,v,K,Q,V$, which describe the algorithmic objects of the transformer architecture, a \textit{numeric superscript}, as the $2$ in $x^2_{i,j}$, shall be used to identify the \textit{layer index} and should not be interpreted as an exponent.

We described the attention mechanism in Definition~\ref{def:transformerlayer}, where the operation performed at layer $\ell \in [L]$ and head $k \in [\kappa]$ is parameterized by the query, key, and value matrices $Q_k^\ell$, $K_k^\ell$ and $V_k^\ell$, respectively. The initial embeddings of the token sequence $z$ is the matrix (equiv., list of vectors) $[z^1_1, \ldots, z^1_N] \in \rR^{d \times N}$, which is the matrix $W_E X$ where $X \in \{0,1\}^{|\vocab| \times N}$ is the one-hot encoding of the tokens $z$. In general, when a token variable is superscripted with 1, we mean the embedded token, i.e. after multiplying by $W_E$. So $x_{i,j}^1$ is the column of $W_E$ corresponding to the token index of $x_{i,j}$.

The attention operation computes the matrix $A_k^\ell \in [0,1]^{N\times N}$. Normally we index this matrix with token indices $i,j \in [N]$, but occasionally we will find it convenient to interpret $A_k^\ell$ is a function on pairs of embedded tokens $z_i^\ell, z_j^\ell$, meaning that overload notation by setting $A_k^\ell(z_i^\ell, z_j^\ell) := A_k^\ell(i,j)$. This is well defined as our token embeddings contain a positional encoding, and we note that this allows us to avoid determining the exact index of the embedded token $x_{i,j}$, which can be cumbersome to describe. Thus we have

\begin{equation}
    A_k^\ell(z_i^\ell, z_j^\ell) = \frac{\exp\left(\inner{z_i^\ell}{z_j^\ell}_{Q_k^\ell (K_k^\ell)^\top}\right)}{\sum_{j'\leq i}\exp\left(\inner{z_i^\ell}{z_{j'}^\ell}_{Q_k^\ell (K_k^\ell)^\top}\right)},
\end{equation}
Finally, $v^\ell_{k, i}$ and $o^\ell_{k,i}$ refer to the ``value vector'' computed at position $i$ for head $k$ and layer $\ell$, and the corresponding output vector:
\begin{align*}
  v^\ell_{k, i} := V^\ell_k z_i^\ell \quad \text{ and } \quad  o^\ell_{k,i} := \sum_{j \leq i} A_k^\ell(z_i^\ell, z_j^\ell) v^\ell_{k,j}.  
\end{align*}
For each $i$ we vertically concatenate all of the outputs $(o^\ell_{k,i})_{k \in \kappa}$ to obtain a tall vector $o^\ell_i \in \rR^{\datt\kappa}$, and the final output $z^{\ell + 1} = [z^{\ell + 1}_1, \ldots, z^{\ell + 1}_N] \in \rR^{d \times N}$ is defined for all $i \in [N]$ as 
\[
    z^{\ell + 1}_i := z^{\ell}_i + \gelu(W_O^\ell o^\ell_i).
\]
In our constructions in the remainder of this appendix, we make a number of simplifications for convenience. Let us state these here, and argue why these are acceptable without loss of generality.
\begin{enumerate}
    \item We often assume that $\kappa = 1$ and we drop the reference to head $k$. Similarly, when $\ell$ is omitted it should be clear from context.
    \item We often implicitly assume that either there is no ``skip connection,'' thus $z^{\ell + 1}_i = \gelu(W_O^\ell o^\ell_i)$, and often go further and assume $z^{\ell + 1}_i = o^\ell_i$, avoiding the $\gelu$ transformation. While we did not use this feature in Definitions~\ref{def:transformerlayer} and~\ref{def:transformer}, in most transformer architectures there is an additional skip connection that makes this possible.
    \item We occasionally refer to the dimension of the embedding $d$ as being different between input and output, i.e. $\din$ and $\dout$. This is for notational ease, and we may do this by padding earlier or later embedding dimensions with 0's.
\end{enumerate}

\section{Proofs of the segmentation results}
\label{app:parsing}

We first give a proof of Theorem~\ref{thm:parsing-sample-main}, and then give the details of the transformer construction from Theorem~\ref{thm:parsing-mech-main}.

\begin{proof}
Let the distribution $U_{\sigma^\star}^{(\sigma', i)} := p_{\sigma'}(\chunk_i | \mathsf{prefix}_{i-1})$ where $\chunk_i$ is the $i$th chunk as parsed by the correct segmentation $\sigma^\star$. Let $\sigma$ be the correct choice of delimiters (we drop the $\star$ superscript for ease of reading). Our goal will be to show that, if we construct the ICL sequence $z$ by sampling i.i.d. $x_1, \ldots, x_n \sim \dist$, computing $y_1, \ldots, y_n$ by applying $y_i = f(x_i)$, and assembling these example/label pairs into a sequence with $\sigma$, then the model estimate $p_\sigma(z)$ is very likely to be much larger than $p_{\sigma'}(z)$ for every alternative $\sigma'\ne \sigma$. What we analyze is the log ratio
\[
    \log \frac{p_{\sigma}(z)}{p_{\sigma'}(z)} = \log \frac{\prod_{i=1}^n U_{\sigma}^{(\sigma, i)}}{\prod_{i=1}^n U_{\sigma}^{(\sigma', i)}} 
    = \sum_{i=1}^n \log \frac{ U_{\sigma}^{(\sigma, i)}}{U_{\sigma}^{(\sigma', i)}},
\]
which we aim to show is very likely to be positive. We do this by converting the above to a martingale sequence. Let $\mu_i := \E_{x_i \sim \dist}\left[ \log \frac{ U_{\sigma}^{(\sigma, i)}}{U_{\sigma}^{(\sigma', i)}} \, \vert \, \mathsf{prefix}_{i-1} \right]$, and observe that $\mu_i > c$ according to Assumption~\ref{assm:parser}.
 Now we have that the sequence $\xi_j := \sum_{i=1}^j \left(\log \frac{ U_{\sigma}^{(\sigma, i)}}{U_{\sigma}^{(\sigma', i)}} - \mu_i\right)$ is a martingale.

We note that, since we have a lower bound $\epsilon$ on the model probabilities, we have that $\log \frac{ U_{\sigma}^{(\sigma, i)}}{U_{\sigma}^{(\sigma', i)}}$ falls within the range $[\log(\minprob), \log(1/\minprob)]$. We can then apply Azuma's Inequality to see that
\begin{eqnarray*}
    \P_{x_{1:n} \sim \dist}\left( \sum_{i=1}^n \log \frac{ U_{\sigma}^{(\sigma, i)}}{U_{\sigma}^{(\sigma', i)}} < 0 \right)
    & = & 
    \P\left( \xi_n < -\sum_{i=1}^n \mu_i \right)\\
    & \leq & 
    \P\left( \xi_n < - nc \right) \\
    & \leq & \exp\left( \frac{-n c^2}{8 \log^2(1/\minprob)} \right).
\end{eqnarray*}
Setting $n$ as in the Theorem statement ensures that the right hand side above is smaller than $\frac{\delta}{|\delims|^2}$. Now if we take a union bound over all possible choices of $\sigma' \ne \sigma$ ensures that
\[
    P(\exists \sigma' \ne \sigma : p_{\sigma'}(z) > p_{\sigma}(z)) < \delta
\]
and thus we are done.
\end{proof}

\subsection{Segmenting example and label delimiters}
We show how to identify example and label delimiters using one head per combination of an example delimiter $\delim_{e}$ and label delimiter $\delim_{l}$ in $\delims\times\delims$. To simplify notation and avoid subscripts, we first focus on one such head for a fixed delimiter pair $(\delim_e,\delim_l)$. We assume that the input to the transformer consists of the vector $z_i^1 = \begin{pmatrix}\tilde z^1_i\\i\\ 1-\ind(z_i = \delim_e)\\1-\ind(z_i = \delim_l) \end{pmatrix}$, where $\tilde z^1_i\in\R^d$ is the (pre-trained) encoding of $z_i$ which we augment for convenience.

\paragraph{First transformer layer:} The goal of the first layer is to take the input at index $i$ and map it to an output $o_i^1 = \begin{pmatrix}z^1_i\\i\\m_e(i)\\1-\ind(z_i = \delim_l) \end{pmatrix}$, where $m_e(i)$ is the largest index $j < i$ such that $z_j = \delim_e$. Let $Q^1$ and $K^1$ matrices in this head be such $Q^1(K^1)^\top$ is $0$ on the first $d$ coordinates and $0$ on the last coordinate.
The remaining $2\times 2$ coordinates are specified as sending $\langle \theta,\theta' \rangle_{Q^1(K^1)^\top(d+1:d+2, d+1:d+2)} = \gamma(\theta'(1)\theta(2) -\theta(1)\theta'(2))$ for any $\theta,\theta'\in \R^2$ for $\gamma \to \infty$. That is the attention weights act as a selector of the $m_e(i)$ as
\begin{align*}
\langle z_i^1, z_j^1 \rangle_{Q^1(K^1)^\top} = \gamma(j(1-\ind(z_i=\delim_e)) - i(1-\ind(z_j=\delim_e))).
\end{align*}
Consequently the soft attention weights act like hard attention and we get that $A^1(z_i,z_j) = \ind(j = m_e(i))$. We further define the value tokens to be $v_j^1 = j$, so that $o_{i}^1 = m_e(i)$. 
Using the skip connection, we further augment the input to the second layer as $z_i^2 = \begin{pmatrix}z^1_i\\i\\m_e(i)\\1-\ind(z_i = \delim_l)\end{pmatrix}$.
\paragraph{Second transformer layer:} The second transformer layer is constructed in a similar way as the first layer, however, it's goal is to extract $m_l(i)$. We are going to assume that each example token can also be treated as a label token (e.g. by adding the indicator of example delimiter to the indicator of label delimiter in the previous layer), so that in the case that there are no label tokens between two example tokens we have $m_e(i) = m_l(i)$. Further, for any $x_i$ it also holds that $m_e(i) = m_l(i)$ and for any $y_i$ it will hold that $m_l(i) > m_e(i)$. This will allow us to distinguish between $x_i$ tokens and $y_i$ tokens, which is important for computing the necessary probabilities for the proof of Theorem~\ref{thm:parsing-sample-main}. We also extend the input $z_i^3$ to contain the following
\begin{align*}
    z_i^3 = \begin{pmatrix}z^1_i\\i\\m_e(i)\\m_l(i)\\m_l(i)m_e(i)\\m_l^2(i)\\m_e^2(i)\\m_l^2(i)m_e(i)\\m_e^2(i)m_l(i)\\0\end{pmatrix}.
\end{align*}

This is easily accomplished by the MLP at in the second layer, following the attention.

\paragraph{Third attention layer:} In the third attention layer we are going to check for consistency of the positions on all separator tokens. Suppose that we have tokens $z_i^3$ and $z_{i'}^3$ s.t. $i\leq i'$. Then the following must be satisfied:
\begin{align}
\label{eq:soundness_cond}
    m_e(i) = m_e(i') \implies m_l(i) = m_l(i') \text{ or } m_l(i) = m_e(i),
\end{align}
that is there can not be more than one label token between two example tokens. To check this consistency we use an attention head for $m_l(i) \leq m_l(i')$.
The attention between $z_i, z_i'$ is computed as
\begin{align*}
    \langle z_{i'}^3, z_{i}^3 \rangle_{Q^3(K^3)^\top} &= \gamma(m_l(i') - m_l(i))\left(m_e(i) - m_e(i') + \frac{1}{2}\right)(m_l(i) - m_e(i))\\
    &+ \frac{\gamma}{n}(i-i')
\end{align*}
for $\gamma \to \infty$, where $n$ is the max sequence length.
Note that by construction it holds that $m_e(i') \geq m_e(i)$ so that this inner product is only $\infty$ if the following hold together $m_l(i') > m_l(i)$, $m_l(i) > m_e(i)$ and $m_e(i) = m_e(i')$. $m_e(i) = m_e(i')$ implies that $i$ and $i'$ are part of the same example, $m_l(i) > m_e(i)$ implies that $i$ is part of the answer sequence for that example and $m_l(i') > m_l(i)$ implies that there is $\lsep$ between token $z_{i'}$ and token $z_{i}$.
The value vectors $v_{i}^3 \in \R^{d+9}$ are set as $v_i^3 = ie_{d+9}$. The resulting output of the attention layer is now $o_{i'}^3 = i' e_{d+9}$ iff the condition in Equation~\ref{eq:soundness_cond} are met, otherwise $o_{i'}^3 = i e_{d+9}$ for some $i$ where the condition is violated. Next, we describe how the MLP acts on $o_{i'}^3 + z_{i'}^3$ (obtained using skip connection) as follows 
\begin{align*}
    o_{i'}^3 + z_{i'}^3 = \begin{pmatrix}z^1_{i'}\\i'\\m_e(i')\\m_l(i')\\m_l(i')m_e(i')\\(m_l(i'))^2\\(m_i(i'))^2\\(m_l(i'))^2m_e(i')\\(m_e(i'))^2m_l(i')\\\iota\end{pmatrix} \to \begin{pmatrix} z^1_{i'}\\ i'\\ \max(m_e(i'), m_l(i'))\\ \mathbf{1}(\iota - i' \geq 0) \end{pmatrix} =: z_{i'}^4
\end{align*}

\paragraph{Fourth transformer layer:} The input to the third layer effectively gives a mapping from the current token to the proposed start of an example and most recent label under the delimiters being considered. Since our objective~\eqref{eq:parsing-obj} evaluates candidate delimiters in terms of the probabilities of the segmented sequences under a pre-trained distribution $p_0$, we need to map the tokens $z_i^4$ to the appropriate conditional probabilities related to the objective. In fact, we assume that the pre-training process provides us with a transformer which can do this mapping, as we state formally below. 

\begin{assumption}[Conditional probability transformer]
There exists a transformer which takes as input a token sequence $\zeta_i, \ldots, \zeta_T$, with each token $\zeta_i = \begin{pmatrix} z^1_i, i, j, \delta_{err}\end{pmatrix}$ for $z^1 \in \R^d, i,j \in [T], \delta_{err},\delta_{delim}\in \{0,1\}$. It produces, $\ln p_i$, for each $\zeta_i$ s.t. $\delta_{err}= 0$, where
\begin{align*}
   \ln p_i  = 
   \begin{cases}
    \ln p_0(z_i|z_{i-1},...,z_{j+1}) \text{ if } j < i-1\\
    \ln p_0(z_i) \text{ if } j = i-1\\
    0 \text{ if } j = i.
   \end{cases}
\end{align*}
Further if $\delta_{err}=1$ it returns $\ln \minprob$.
\label{assm:cond-prob-transform}
\end{assumption}

We note that since the basic language models are trained to do next word prediction through log likelihood maximization, this is a very reasonable abstraction to assume from pre-training. As a result, we assume that the fourth layer produces $z^5_i = \ln p_i$. 

Assumption~\ref{assm:cond-prob-transform} allows us to compute conditional probabilities of sequences according to their segmentation in the following way. Consider the sequence
\begin{align*}
 \begintoken z_1, z_2, \lsep, z_3, z_4 \esep, z_5, \lsep \endtoken.
\end{align*}
What we feed to the transformer from Assumption~\ref{assm:cond-prob-transform} is 
\begin{align*}
    \begin{pmatrix} z^1_1\\ 1\\ 0\\ 0\end{pmatrix}, \begin{pmatrix} z^1_2\\ 2\\ 0\\ 0\end{pmatrix}, \begin{pmatrix} \lsep\\ 3\\ 3\\ 0\end{pmatrix}, \begin{pmatrix} z^1_3\\ 4\\ 3\\ 0\end{pmatrix}, \begin{pmatrix} \esep\\ 5\\ 5\\ 0\end{pmatrix}, \begin{pmatrix} z^1_5\\ 6\\ 5\\ 0\end{pmatrix}, \begin{pmatrix} \lsep\\ 7\\ 7\\ 0\end{pmatrix}.
\end{align*}
The transformer, respectively, computes
\begin{align*}
    p_1 = p_0(z_1), p_2 = p_0(z_2|z_1), p_3 = 1, p_4 = p_0(z_3), p_5 = 1, p_6 = p_0(z_5), p_7= 1.
\end{align*}

\paragraph{Fifth transformer layer:} Note that so far we have acquired the conditional probabilities of individual tokens, when conditioned on a prefix. Further, these conditional probabilities have the following properties. If $i$ is a delimiter then $p_i = 1$. If $i$ is such that $i-1$ is a delimiter then $p_i = p_0(z_i)$, that is the marginal of $z_i$ is returned and finally $p_i = \minprob$ for some small $\minprob > 0$ if there is some inconsistency in the token segmentation. Note that this ensures that inconsistent token segmentations will have small probability. The fifth transformer layer just assigns uniform
attention with the $0$ matrix for $QK^T$ and use $v_j = \begin{pmatrix}\ln p_j\\1 \end{pmatrix}$. This results in an output $\sum_{j=1}^i (\ln p_j)/i$ at token $i$ and $z_i^6 = \begin{pmatrix}\sum_{j=1}^i (\ln p_j)/i\\i\end{pmatrix}$.
Finally, we use two MLP layers to send $z_i^6 \to z_i^6(1)\times z_i^6(2) = \sum_{j=1}^i (\ln p_j)$.
\begin{lemma}
For any token pair $\sigma$ the output at token $k$ of the transformer at attention head associated with $\sigma$ satisfies the factorization in Equation~\ref{eq:sigma_segmentation_distr}.
\end{lemma}
The above lemma is a direct consequence of the construction of the transformer. We have already seen how the fourth layer acts on the sequence
\begin{align*}
 \begintoken z_1, z_2, \lsep, z_3, z_4 \esep, z_5, \lsep \endtoken.
\end{align*}
The final layer of the transformer will now output for token $7$ the sum
\begin{align*}
    \sum_{i=1}^7 \ln p_i &= \ln \Pi_{i=1}^7 p_i = \ln \left(p_0(z_1)\times p_0(z_2|z_1)\times p_0(z_3) \times p_0(z_4 | z_3)\times p_0(z_5)\right) \\
    &= \ln(p_0(z_1, z_2)p_0(z_3,z_4) p_0(z_5)),
\end{align*}
which is precisely the factorization in Equation~\ref{eq:sigma_segmentation_distr} which is also used in the proof of Theorem~\ref{thm:parsing-sample-main}.

\paragraph{Final layer to select across delimiters:} Finally, the network implements the objective~\eqref{eq:parsing-obj} for a particular delimiter in one head. By using the MLP to implement maximization across heads from the concatenated output values, we can identify the optimal delimiter.

\section{Learning the the 1-sparse tokenized regression task}
\label{sec:token_task}

Here we give the construction of a transformer mechanism and sample complexity for the $1$-sparse tokenized regression task, to build intuition for general the $s$-sparse case. We start with the mechanism before giving the sample complexity result.

\subsection{Transformer mechanism for $1$-sparse tokenized regression}
\label{sec:token_task_mech}

Recall that the $1$-sparse tokenized regression task is defined by a vector in $x \in \R^m$ and the hypothesis class $\cF$ consists of all basis vectors $\{e_i\}_{i=1}^m$ in $\R^m$. Each instance of the task is defined by fixing a vector $e_f \in \{e_i\}_{i=1}^m$. The labels for the task are $y_i = \langle x_i, e_f^\star \rangle$ for some unknown $f^\star \in \cF$. The $i$-th element of the sequence given to the transformer for the task is $(x_{i,1},\ldots, x_{i,m}, y_i)$, where $x_{i,j}$ denotes the $j$-th coordinate of the vector $x_i$. We now describe how ICL can learn the above task, by using $1$ head per layer and $4$ attention layers.

We begin by stating a useful lemma which will allow us to set attention weights between any two tokens $z_i, z_j$ to $0$ by only using positional embedding dependent transformations.
\begin{lemma}
\label{lem:attn_weights_choice}
For any given token embeddings $\{\bar z_{i}\}_{i\in[mn]}\in \R^d$, there exist embeddings $\{z_{i}^\ell\}_{i\in [mn]} \in \R^{d+mn}$ which only depend on the positions $i \in [mn]$ such that for any subsets $S,S' \subseteq [mn]$ (not necessarily disjoint) we have $\langle z_{i}^\ell,  z_{j}^\ell \rangle_{Q_k^\ell (K_k^\ell)^\top} = c_1 + c_2\inner{\bar z_i}{\bar z_j} \in \R\bigcup\{\infty, -\infty\}$ for all $i\in S, j\in S'$ and $\inner{z^\ell_j}{z^\ell_j} = c_2 \inner{\bar z_i}{\bar z_j}$ otherwise.
\end{lemma}
\begin{proof}
We embed $z^\ell_{i}$ into $\R^{d+mn}$ in the following way
\begin{align*}
    z^\ell_{i}(t) = 
    \begin{cases} 
    \bar z_{i}(t)& 1\leq t\leq d\\
    \sqrt{|c_1|}& t=i\\
    0&\text{otherwise}
    \end{cases}
\end{align*}
We now take $Q_{k}^\ell \in \R^{(d+mn)\times(d+1)}$ to contain the $c_2$-scaled identity mapping in the first $d$ columns. The $d+1$-st column is set to $0$ in the first $d$ rows and as $Q_{k}^\ell(t, d+1) = \mathbf{1}(t-d\in S)$ . Similarly we set the first $d$ columns of $K_{k}^\ell \in R^{(d+mn)\times(d+1)}$ to the identity and $K_{k}(t, d+1) = \mathbf{sgn}(c_1)\mathbf{1}(t-d\in S')$ and $0$ otherwise.
\end{proof}

Recall that Assumption~\ref{assm:orth} gives approximate orthogonality of the embeddings in the first layer of the transformer. We note that the assumption is not hard to satisfy, e.g., one can use the "bucket" construction described right after the assumption in Section~\ref{sec:hyp-learning}. Other possible embeddings include a Random Fourier Feature approximation to a Gaussian kernel with appropriate bandwidth. In fact, since the lemma only requires only approximate orthogonality, we can further assign each bucket in the bucket embedding to a random Gaussian vector in $O(\tau^2)$ dimensions and obtain the same result with dimension $\tilde O(\tau^2)$ with high probability. We now give the construction of a transformer for this learning task, assuming the availability of such an embedding satisfying Assumption~\ref{assm:orth} in $R^{d_\epsilon}$ for some $d_\epsilon$.

\paragraph{Construction of the first transformer layer.} We now specify the first transformer layer, by specifying the value vectors. For the query and key matrices we use Lemma~\ref{lem:attn_weights_choice} to argue that there exists a setting such that for any embedding in dimension $\din = d_\epsilon + mn$ satisfying Assumption~\ref{assm:orth} we have an embedding for which the attention weights can be computed as
\begin{align}
    &A^1(y_i, x_{i,j}) = \frac{\exp(\langle y_i^1, x_{i,j}^1 \rangle)}{\exp(\langle y_i^1, y_i^1 \rangle) + \sum_{s=1}^m\exp(\langle y_i^1, x_{i,s}^1 \rangle)}, ~~A^1(y_i, x_{i',j}) = A^1(y_i, y_{i'}) = 0 \text{~for $i'\ne i$ and}\\&A^1(x_{i,j}, x_{i',j'}) = \ind(i = i', j = j'),~~A^1(x_{i,j}, y_{i'}) = 0.
    \label{eq:first-attn}
\end{align}
That is the attention weights only depend on inner products between tokens within an example, and each $x_{i,j}$ token only attends to itself, while the $y_i$ attends to all the tokens in $x_i$ with different weights. Next, the value transformation is chosen as the map which sends any $x_{i,j}^1 \in \R^{\din}$ to $v_{i,j} \in \R^{\din + m}$, where the first $\din$ coordinates of $v_{i,j}$ are equal to $x_{i,j}^1$ and the remaining $m$ coordinates are equal to the basis vector $e_j \in \R^m$. We will shortly see how ICL learns a hypothesis consistent with all examples but the main idea is that the consistent hypothesis will have the highest weight within the last $m$ coordinates of the vector $\sum_{i=1}^n \sum_{j=1}^m A^1(y_i, x_{i,j}) v_{i, j}^1$, that is the sum of the value vectors associated with each of the answer tokens $y_i, i\in[n]$, after the attention layer has been applied.

Let $v^1_{i,m+1} = 0$.
Now the outputs $o^1_{i,j} \in \R^{\din + m}$ of self-attention at the first layer can be written as 
\begin{align}
    o_{i,j}^1 = v_{i,j}^1,\quad
    o_{i, m+1}^1 = \sum_{j=1}^{m} A^1(y_{i}^1, x_{i,j}^1)v_{i,j}^1 + A^\ell(y_i^1, y_i^1)v_{i, m+1}^1.
    \label{eq:first-out}
\end{align}
We now argue that output from the $y_i$ tokens identifies all positions $j$ such that $|x_{i,j} - y_i| \leq \epsilon$.
\begin{lemma}[Restatement of Lemma~\ref{lemma:margin}]
Given an example $i$, let $\cJ_i = \{j~:~|x_{i,j} - y_i| \leq \epsilon\}$, and let $f^\star \in [m]$ be such that $y_i = x_{i, f^\star}$. Under Assumption~\ref{assm:orth}, for any $m \geq 2$, the output of the first layer satisfies that for any $i\in[n]$:
\[
    o^1_{i, m+1}(\din + f^\star) \geq \max_{j \in [m]\setminus\cJ_i} o^1_{i, m+1}(\din + j) + \frac{e}{4(m+1)}.
\]
\end{lemma}

\begin{proof}
    By Equation~\ref{eq:first-out}, we know that under Assumption~\ref{assm:orth}, for any sample $i$ and index $j \notin \cJ_i$, $A^1(y^1_i, x^1_{i,j}) \leq e^{1/(2\tau)}/Z_i \leq e^{1/2}/Z_i$, where $Z_i$ is denominator in Equation~\ref{eq:first-attn}. On the other hand, $A^1(y_i, y_i) = A^1(y_i, x_{i, f^\star}) = e/Z$, under Assumption~\ref{assm:orth}. Hence, we see that 
    \begin{align*}
        o^1_{i, m+1}(\din + f^\star) - \max_{j \in [m]\setminus\cJ_i} o^1_{i, m+1}(\din + j) \geq \frac{e - e^{1/2}}{Z_i} \geq \frac{e}{4Z_i},
    \end{align*}
    where the second inequality uses $m \geq 2$. Since $Z_i \leq e(m+1)$ under our definition of the attention weights, this yields the lemma.
\end{proof}
We assume that the following MLP layer acts as the identity and $x_{i,j}^2 = v_{i,j}^2 = o_{i,j}^1, \forall j\in[m+1]$, where we take $x_{i,m+1}^2 \equiv y_i^2$. Notice that the embedding dimension changes from $\din$ in the first layer to $\din + m$ for the second layer.

\paragraph{Second transformer layer.}
For inference we set the attention weights in the second attention head in the following way, which is permitted by Lemma~\ref{lem:attn_weights_choice}:
\begin{align}
    A^2(x_{i,j}^2, x_{i',j'}^2) = \ind(i = i', j = j'), A^2(x_{i,j}^2, y_{i'}^2) = 0, A^2(y_i^2, x_{i',j}^2) = 0, \forall i'\leq i,~  A^{2}(y_i^2, y_t^2) = 1/i, \forall t\leq i.
    \label{eq:second-attn}
\end{align}
In words, the $x_{i,j}$ tokens only attend to themselves again, while $y_i$ attends uniformly to all the previous labels, including itself.
This construction implies
\begin{align}
    o^2_{i, j} = v^2_{i,j} = o^1_{i,j} = 
    \begin{pmatrix}
    x_{i,j}^1 \\
    e_{j}
    \end{pmatrix}, j \in [m],
    o^2_{i, m+1} =  \frac{1}{i}\sum_{t=1}^i v^2_{t, m+1} =  \frac{1}{i}\sum_{t=1}^i o^1_{t, m+1}.
    \label{eq:second-out}
\end{align}

We assume that the following MLP acts as the identity mapping on the first $d$ coordinates of the value vectors and then sends the remaining $m$ coordinates to the basis vector corresponding to the index with highest value. That is, $\text{MLP}^2 o^2_{i, m+1} = (v\,\, e_{f_i})$, with $v \in \R^{\din}$ equal to the first $\din$ coordinates of $o^2_{i, m+1}$, and $f_i = \argmax_{j \in [m]} o^2_{i, m+1}(j)$. Ties are broken arbitrarily but consistently. 

In particular, together with the construction of the attention weights at the previous layer we have 
\begin{align}
    x^3_{i, j} = 
    \begin{pmatrix}
    x_{i,j}^1\\
    e_{j}
    \end{pmatrix}, j \in [m],\quad
    y^3_{1, i}(d+1:d+m) = e_{f_i} \in \R^m,\quad f_i = \argmax_{j \in [m]} \frac{1}{i}\sum_{t=1}^i o^1_{t, m+1}(j).
    \label{eq:third-in}
\end{align}

In the next section, we show that this aggregation across examples followed by the maximization selects some coordinate such that $|x_{i,j} - y_i| \leq \epsilon$ for all examples $i$ in the context, and that any such hypothesis has a small prediction error on future examples in this task. That is, the first two layers identify an approximately correct hypothesis for the task.

\paragraph{Inference with learned hypothesis.} Finally we explain how to apply the returned hypothesis $f_{i}$ to the next example. This will also describe how to do inference with the hypothesis $f_n$ on the $n+1$-st example. The attention pattern required here is a bit different in that each $x_{i,j}$ only attends to the previous label $y_i$ and itself. $Q_{1}^3(K_{1}^3)^\top$ only acts on coordinates $[d+1:d+m]$ as the identity and sends everything else to $0$. In particular, the unnormalized attention weights are
\begin{align*}
    \exp(\langle x_{i+1, j}^3, y_{1, i}^3  \rangle_{Q^3(K^3)^\top}) &=   \exp(\langle e_j, e_{f_i} \rangle) = \exp(\mathbf{1}(f_i = j))\\
    \exp(\langle x_{i+1, j}^3, x_{i+1, j}^3  \rangle_{Q^3(K^3)^\top}) &= \exp(\langle e_j, e_j \rangle) = e\\
    \exp(\langle x_{i+1, j}^3, x_{i+1, j'}^3  \rangle_{Q^3(K^3)^\top}) &= 0
\end{align*}

This results in attention values

\begin{align}
    A^3(x^3_{i+1,j}, y^3_{i,i}) = \left\{\begin{array}{cc} \frac{1}{2} & \mbox{$f_i = j$}\\ \frac{1}{1+e} < \frac{1}{2} & \text{otherwise}\end{array}\right.,~~A^3(x^3_{i,j}, x^3_{i,j}) = 1 - A^3(x^3_{i+1,j}, y^3_{i,i}).
    \label{eq:third-attn}
\end{align}
The remaining attention outputs are not used and hence not specified here. The value vectors in $\R^2$ are set as
\begin{align*}
    v_{i, j}^3 = 2\begin{pmatrix}x_{i, j}\\0\end{pmatrix}, 
    v_{i, m+1}^3 = 2\begin{pmatrix}0\\1\end{pmatrix}. 
\end{align*}
Notice that this requires access to the raw input token, which can be done by either providing a skip connection from the inputs, or by carrying the input token as part of the embedding through all the layers at the cost of one extra embedding dimension.
As a result, for the index $f_i$ selected at the end of example $i$, we have that
\begin{align*}
   o_{i+1,f_i}^3 = \begin{pmatrix} x_{i+1,f_i}, 1\end{pmatrix},
\end{align*}
and for any other $j \in [m]$ we have $o_{i+1,j}^3 < 1/2$ in the second coordinate. The next MLP layer thresholds the second coordinate of $o_{i+1,j}^3$ so that 
\begin{align*}
    x_{i+1,j}^4 = \begin{pmatrix} x_{i+1, j}, \mathbf{1}(\text{$j$ is consistent with $f^\star$ up to example $i$}) \end{pmatrix}.
\end{align*}
The final attention and MLP layers are used to copy any $x_{i+1, j}^4$ to the $m$-th token of the $i+1$-st example, so that the transformer outputs the prediction at the end of each example sequence. This can be done by using the \texttt{mov} function described in Section 3.1 of \citet{akyurek2022learning}.

A summary of this construction can be found in Figure~\ref{fig:icl_example}.

\subsection{Sample complexity}
Let the target hypothesis be $f^\star$, that is we assume, $y_i = f^\star(x_i), \forall i\in[n]$. We are going to analyze the error of the hypothesis returned by ICL after $m$ examples. From Lemma~\ref{lemma:margin}, we know that the true hypothesis $f^\star$ has a large value in the output of the first layer, in the coordinate $d+f^\star$, at each example $i$. Suppose our construction identifies the hypothesis to make the prediction with, after seeing $i$ examples. Then if $f_i$ makes an incorrect prediction (that is $|y_{i'} - f_i(x_{i'})| \geq \epsilon$ for some $i' \leq i$) on even one of these $i$ examples, the output in coordinate $d+f_i$ is guaranteed to be smaller than in $d+f^\star$ by Lemma~\ref{lemma:margin}. Consequently, the hypothesis $f_n$ returned after $n$ examples is guaranteed to have an error at most $\epsilon$ on each of the $n$ examples in context. We now show that this implies a risk bound on the hypothesis $f_n$.

\begin{lemma}
\label{lem:risk_bound}
Let $p_n = \P(|f_n(x) - f^\star(x)| \leq \epsilon)$ be the probability of the returned hypothesis deviating from $f^\star$ by more than $\epsilon$ on any example. Then with probability $1-\delta$ it holds that $p_n \geq 1 - \frac{20\log(m/\delta)}{3n}$.
\end{lemma}
\begin{proof}
Fix any hypothesis $f$. Let $X_i$ denote the Bernoulli random variable indicating the event that $|f(x_i) - f^\star(x_i)| \leq \epsilon$ and $p_f = \P(|f(x) - f^\star(x)| \leq \epsilon)$.
We compute the probability that this hypothesis is potentially returned by the transformer which is equivalent to the event that $\sum_{i=1}^n X_i \geq n$.
\begin{align*}
    \P(\sum_{i=1}^n X_i \geq n) = \P\Bigg(\sum_{i=1}^n X_i &\geq np_f + 2\sqrt{np_f(1-p_f)\log(1/\delta)} + \frac{4}{3}\log(1/\delta)\\
    &+ n(1-p_f) - 2\sqrt{np_f(1-p_f)\log(1/\delta)} - \frac{4}{3}\log(1/\delta)\Bigg) \leq \delta,
\end{align*}
as long as $n(1-p_f) - 2\sqrt{np_f(1-p_f)\log(1/\delta)} - \frac{4}{3}\log(1/\delta) \geq 0$.
We note that Cauchy-Schwartz implies
\begin{align*}
    n(1-p_f) - 2\sqrt{np_f(1-p_f)\log(1/\delta)} - \frac{4}{3}\log(1/\delta) &\geq \frac{n(1-p_f)}{2} - \frac{4}{3}\log(1/\delta) - 2p_f\log(1/\delta)\\
    &\geq \frac{n(1-p_f)}{2} - \frac{10}{3}\log(1/\delta)
\end{align*}
Finally, we note that $p_f \leq 1 - \frac{20\log(1/\delta)}{3n}$ implies $\frac{n(1-p_f)}{2} - \frac{10}{3n}\log(1/\delta) \geq 0$. Taking a union bound over all possible $f$ and applying with $f = f_n$, so that $p_{f_n} = p_n$ completes the proof.
\end{proof}

\begin{theorem}
\label{thm:sample_complexity_itt}
For any $\epsilon > 0$ there exists an embedding of $x_{i,j}, y_i, \forall i\in[n], j\in[m]$ in $\R^{O(1/\epsilon)}$ such that for $n = \Omega(\log(m/\epsilon)/\epsilon)$ it holds that 
$\rE[|f_n(x) - f^\star(x)|]\leq 2\epsilon$, where $f_n$ is the hypothesis returned by ICL.
\end{theorem}
\begin{proof}
We use the embedding into $\lceil1/\epsilon\rceil$ buckets, as mentioned in the previous section together with the construction of the transformer to satisfy the conditions of Lemma~\ref{lem:risk_bound}. Conditioning on the good event, $A$, in Lemma~\ref{lem:risk_bound} implies that $\P(|f_n(x) - f^\star(x)|>\epsilon|A) \leq \epsilon$ and so under $A$, we have
\begin{align*}
    \rE|f_n(x) - f^\star(x)| \leq \epsilon p_n + 1-p_n \leq 2\epsilon,
\end{align*}
where the second inequality follows from our condition on $n$.
\end{proof}

\section{$s$-sparse Tokenized Regression}
\label{app:s-sparse}
In this section we study the general $s$-sparse case defined in Definition~\ref{defn:gen-reg-task}. Recall that the hypothesis class now consists of $f = (j_1,\ldots,j_s) \in [m]^s$, that is each hypothesis selects $s$ out of the $m$ coordinates of $x$.
We begin by making the following simple observation under Assumption~\ref{assm:orth}: if $j \in f^\star$ then for any $i$ it holds that $\langle x_{i,j}^1, y_i^1\rangle \geq 3/4$, while if $j$ is not part of a consistent policy then we have $\langle x_{i,j}^1, y_i^1 \rangle \leq 1/4$.
\begin{lemma}
\label{lem:sep-s-sparse}
Under Assumption~\ref{assm:orth} with $\tau \geq 2s$ we have that for any $\mathcal{C} \subseteq f^\star$ and, $j \in f^\star\setminus\mathcal{C}$, we have $\langle x_{i,j}^1, y_i^1 - \sum_{j'\in\mathcal{C}} x_{i,j'}^1\rangle \geq \frac{3}{4}$, while if $|x_{i,j} - x_{i,j'}| \geq \epsilon$ for all $j' \in f^\star$ then we have $\langle x_{i,j}^1, y_i^1 - \sum_{j'\in\mathcal{C}} x_{i,j'}^1\rangle \leq \frac{1}{4}$.
\end{lemma}

\begin{proof}
Since $y_i^1 = \sum_{j\in f^\star} x^1_{i,j}$, for any $j \in f^\star\setminus\cC$, we have 
\begin{align*}
    \inner{x_{i,j}^1}{\sum_{j \in f^\star\setminus \cC} x^1_{i,j}} &\geq \inner{x^1_{i,j}}{x^1_{i,j}} - \frac{s}{2\tau} \geq \frac{3}{4},
\end{align*}
where the first inequality follows from Assumption~\ref{assm:orth}, since any token which does not have an inner product of $1$ with $x_{i,j}$ has the inner product at least $-1/(2\tau)$. The last equality follows from the precondition $\tau \geq 2s$ in the lemma. On the other hand, for any token $j$ which is not $\epsilon$-close to any token in $f^\star$, the inner product is at most $s/(2\tau)$ by a similar argument, which completes the proof.
\end{proof}

We proceed to give a construction which will use $O(m)$ layers with one head per layer. The idea behind the construction is to learn each coordinate of a single consistent hypothesis in $\mathcal{F}$. We note that it is not possible to directly take the approach in the index token task to learn each coordinate in $f^\star$ independently now, unless there is a unique consistent hypothesis with high probability. As described in Section~\ref{sec:hyp-learning}, we follow an iterative deflation approach to avoid this issue.

Suppose that at layer $\ell$ we have learned a set of coordinates of a consistent hypothesis. Denote the subset of the learned coordinates which are equal to $1$ as $\mathcal{C}^{\ell}_i$. The embedding for $y_i^{\ell} \in \R^{d+m}$ then consists of $y_i^1$ in the first $d$ coordinates, and the following holds for the remaining $m$ coordinates. If coordinate $j \in \mathcal{C}^{\ell}_i$, then $y_i^{\ell}(j) = 0$, otherwise $y_i^\ell(j) = -\infty$. The embedding of $x_{i,j}^\ell \in \R^{d+m}$ is as follows. The first $d$ coordinates are again equal to $x_{i,j}^1$, the remaining $m$ coordinates equal the coordinates of $e_j$. The value vectors are set to $v_{i,j}^{\ell} = \begin{pmatrix}-x_{i,j}^1\\-1\end{pmatrix}$ for $j \leq m$ and $v_{i, m+1}^{\ell} = \begin{pmatrix}y_{i}^1\\1\end{pmatrix}$.
The query and key matrices are set to act as the $0$ matrix on the first $d$ coordinates and as the identity on the remaining $m$ coordinates, except for the token associated with $y_i$, so that $\langle y_i^\ell, y_i^\ell \rangle_{Q^\ell (K^\ell)^\top} = 0$. We have the following.
\begin{lemma}
\label{lem:deflation}
There exists a setting for the query, key and value matrices at layer $\ell$ so that given the embeddings $y_i^{\ell}$ and $x_{i,j}^{\ell}, i\in[n], j\in[m]$ it holds that
\begin{align*}
    o_{i, m+1}^{\ell} &= \begin{pmatrix}\frac{1}{|\mathcal{C}^\ell_i|+1}\left(y_{i}^1 - \sum_{j \in \mathcal{C}^{\ell}_i} x_{i, j}^1\right)\\\frac{1}{|\mathcal{C}^\ell_i|+1}\end{pmatrix}\in \R^{d+1}\\
    o_{i, j}^{\ell} &= \begin{pmatrix}-x_{i,j}^1\\-1 \end{pmatrix} \in \R^{d+1}.
\end{align*}
\end{lemma}
\begin{proof}
To show the claim of the lemma we only need to compute the attention weights from the $\ell$-th attention layer.
First, using Lemma~\ref{lem:attn_weights_choice} we can set $A^{\ell}(x_{i,j}^\ell, x_{i, j'}^\ell) = -\mathbf{1}(j=j')$, which, together with the value vector choice, shows that $o_{i, j}^{\ell} = -x_{i,j}^1$.
If $j\not\in \mathcal{C}^\ell_i$ then the construction implies
\begin{align*}
    \langle y_i^\ell, x_{i,j}^\ell \rangle = -\infty.
\end{align*}
Further, using the position embedding of $y_i$ we use Lemma~\ref{lem:attn_weights_choice} to set
\begin{align*}
    \langle y_i^\ell, y_i^\ell \rangle_{Q^\ell (K^\ell)^\top} = 0.
\end{align*}
Finally, we want to ensure uniform weights for all consistent examples in $\cC_i^\ell$ and so we enforce $\langle y_i^1, x_{i,j}^1 \rangle_{Q^\ell(K^\ell)^\top} = 0$ as described in the construction. Thus for any $j\in \mathcal{C}^\ell_i$ we have
\begin{align*}
    A^\ell(y_i^\ell, x_{i,j}^\ell) = \frac{\exp(0)}{\sum_{j \in \mathcal{C}_i^\ell} \exp(0) + \exp(\langle y_i^\ell, y_i^\ell \rangle_{Q^\ell (K^\ell)^\top})} = \frac{1}{|\mathcal{C}_i^\ell|+1}.
\end{align*}
For $j\not\in\mathcal{C}^\ell_i$ we have $ \langle y_i^\ell, x_{i,j}^\ell \rangle = -\infty$ and this implies $A^\ell(y_i^\ell, x_{i,j}^\ell)=0$, which completes the claim of the lemma.
\end{proof}

Lemma~\ref{lem:deflation} shows that we can "deflate" $y_i^1$ by subtracting all consistent coordinates which have been identified so far. Next, we are going to use the construction for the $1$-sparse task on $i$-th example $y_i^{\ell+1} = y_i^1 - \sum_{j \in \mathcal{C}_i^\ell} x^1_{i,j}$ and $x_{i,j}^{\ell+1} = x_{i,j}^1$.
We make a slight modification to the outputs of the $\ell$-th attention layer by setting
\begin{align*}
    o_{i, m+1}^\ell &= \begin{pmatrix}o_{i, m+1}^\ell\\ y_i^\ell(d+1:m)\end{pmatrix}\\
    o_{i,j}^\ell &= \begin{pmatrix}o_{i, m+1}^\ell\\ x_{i,j}^\ell(d+1:m)\end{pmatrix}.
\end{align*}
This can be achieved using the skip-connection and appropriate padding of $o_{i,m+1}$.
However, to simplify the argument we avoid describing this operation.
We assume that the MLP layer after the $\ell$-th attention layer acts on $o_{i,j}^\ell$ in the following way, it sends $o_{i,j}^\ell \to \frac{1}{o_{i,j}^\ell(d+1)} o_{i,j}^\ell$.
Further, it acts on the coordinates corresponding to $y_i^\ell(d+1 :m)$ by sending $-\infty$ to $0$ and $0$ to $-\infty$. This can be done by first adding $1$ to all coordinates, then using a relu to clip all remaining $-\infty$ to $0$, and finally multiply the remaining positive coordinates by $-\infty$ again. This operation is needed to take the complement of $\mathcal{C}^\ell_i$ so that all consistent coordinates which have already been added to $\mathcal{C}^\ell_i$ can be removed from consideration. We note that both these operations actually need a 2-layer MLP, however, for simplicity we assume that these are implementable by the MLP layer following the attention layer.

We now describe the inputs $x_{i,j}^{\ell+1}$ and $y_{i}^{\ell+1}$ to the $\ell+1$-st transformer layer:
\begin{equation}
\label{eq:1sparse_reduction}
    \begin{aligned}
        x_{i,j}^{\ell+1}(1:d) &= x_{i,j}^1\\
        x_{i,j}^{\ell+1}(d+1:m) &= e_{j}\\
        y_i^{\ell+1}(1:d) &=  y_{i}^1 - \sum_{j \in \mathcal{C}^{\ell}_i} x_{i, j}^1\\
        y_i^{\ell+1}(d+1:m)(j) &= -\infty\mathbf{1}(j\in\mathcal{C}^\ell_i).
    \end{aligned}
\end{equation}
The above implies for all $j\in \mathcal{C}^\ell_i$ we have $\langle x_{i,j}^{\ell+1}, y_i^{\ell+1}\rangle = -\infty$ and otherwise $\langle x_{i,j}^{\ell+1}, y_i^{\ell+1}\rangle = \langle x_{i,j}^1, y_{i}^1 - \sum_{j \in \mathcal{C}^{\ell}_i} x_{i, j}^1\rangle$. Finding a consistent coordinate is now equivalent to recovering a consistent hypothesis for the $1$-sparse task, which we know how to do using exactly two attention layers as described previously.
\begin{lemma}
\label{lem:ell3_out}
Applying the first two layers of the $1$-sparse task from Section~\ref{sec:token_task_mech} to $x_{i,j}^{\ell+1}, j\in [m], y_{i}^{\ell+1}$ as defined in Equation~\ref{eq:1sparse_reduction} yields:
\begin{align*}
    o_{i,j}^{\ell+3} &=  x_{i,j}^{\ell+1} = \begin{pmatrix} x_{i,j}^1\\e_j\end{pmatrix}, j\in[m]\\
    o_{i, m+1}^{\ell+3} &= \frac{1}{i}\sum_{t=1}^i\sum_{j=1}^m A^{\ell+1}(y_{t}^{\ell+1}, x_{t,j}^{\ell+1})x_{t,j}^{\ell+1} = \frac{1}{i}\sum_{t=1}^i\sum_{j\in[m]\setminus\mathcal{C}_t^\ell} A^{\ell+1}(y_{t}^{\ell+1}, x_{t,j}^{\ell+1})x_{t,j}^{\ell+1}.
\end{align*}
\end{lemma}
\begin{proof}
Using the fact that $\langle x_{i,j}^{\ell+1}, y_i^{\ell+1}\rangle = -\infty$ for $j\in \mathcal{C}^\ell_i$ we see that $A^{\ell+1}(y_i^{\ell+1}, x_{i,j}^{\ell+1}) = 0$, which implies the second inequality for $o_{i, m+1}^{\ell+3}$. To argue the first equality and the result for $o_{i,j}^{\ell+3}$ we directly appeal to Equation~\ref{eq:second-out}, together with checking that the separation condition of Lemma~\ref{lemma:margin} is satisfied. This condition is directly implied by Lemma~\ref{lem:sep-s-sparse}.
\end{proof}
Using Lemma~\ref{lemma:margin} we have that for every $j$ selected by some consistent hypothesis $A^{\ell+1}(y_t^{\ell+1}, x_{t,j}^{\ell+1})$ will exceed $A^{\ell+1}(y_t^{\ell+1}, x_{t,j'}^{\ell+1})$, where $j'$ is not selected by any consistent hypothesis. This implies that the maximum coordinate among $[d+1,m]$ of $o_{i, m+1}^{\ell+3}$ will be included in a consistent hypothesis for all $t\leq i$ examples. This implies that applying the second MLP layer from the $1$-sparse task will write a consistent coordinate in $y_i^{\ell+4}(d+1:m)$. Further, Lemma~\ref{lem:ell3_out} implies that this consistent coordinate will not be part of the already fixed coordinates in $\mathcal{C}_i^\ell$. Let this new consistent coordinate be $j^\ell$. We would like to add $j^\ell$ to $\mathcal{C}_i^\ell$. This is done as follows. First we assume that the MLP sets $y_i^{\ell+4}(d+1:m) = -e_{j^\ell}$. Next, we assume access to a skip connection from layer $\ell+1$ so that we can add $y_{i}^{\ell+4}(d+1:m) + y_{i}^{\ell+1}(d+1:m)$. To transform $y_{i}^{\ell+4}(d+1:m) + y_{i}^{\ell+1}(d+1:m)$ to a similar construction used with $y_{i}^\ell$ we first add $1/2$ to every coordinate of $y_{i}^{\ell+4}(d+1:m) + y_{i}^{\ell+1}(d+1:m)$. Next, we use another relu activation on each coordinate. The resulting vector already satisfies that every coordinate $j\in\mathcal{C}_i^{\ell+4}$ is equal to $0$, and every coordinate outside of the set is $\frac{1}{2}$. It remains to multiply the resulting vector by $-\infty$ and add $y_i^1$ to the first $d$ coordinates using a skip connection. All of the above can be done using one additional attention layer, together with an MLP. Since skip connections in the original transformer architecture are only in between consecutive attention layers, we can implement the above by extending the embedding of each $y_i^{\ell+1},\ldots,y_i^{\ell+4}$ to have an additional $d+m$ coordinates in which to store $y_{i}^1$ together with the representation of $\mathcal{C}_i^\ell$.

\paragraph{Applying the learned hypothesis.} The above construction implies that after $L=O(s)$ layers the resulting $y_i^L(d+1:m)$ will contain exactly a set $\mathcal{C}_i^\ell$ of cardinality $s$ which contains only consistent coordinates. Further, using the deflation construction, we can show the following.
\begin{lemma}
\label{lem:output_hypothesis}
After $L = O(s)$ layers it holds that $|\mathcal{C}^{L}_i|=s$ and further, there exists a bijection $b_i$ from $f^\star$ to $\mathcal{C}^L_i$ such that for any $j\in f^\star$, $|x_{t,j} - x_{t, b_i(j)}| \leq \epsilon, t\leq i$. The output $y_i^L \in \R^{d+m}$ is such that $y_i^L(\mathcal{C}^L_i) = 0$ and $y_i^L([m]\setminus\mathcal{C}^{L}_i)= -\infty$.
\end{lemma}
\begin{proof}
For the first part of the lemma we begin by showing that for any $j' \in \mathcal{C}^L_i$, there exists a $j \in f^\star$ such that $|x_{i,j} - x_{i,j'}| < \epsilon$. Suppose that this does not hold true, i.e., there is some $j'$ such that for all $j \in f^\star$, for which $|x_{t,j} - x_{t,j'}| \geq \epsilon$ for some $t\leq i$. Lemma~\ref{lem:sep-s-sparse} implies that $\langle x_{t,j'}^1, y_{t}^1 \rangle \leq \frac{1}{4}$. On the other hand if $j' \in \mathcal{C}^L_i$ then the construction implies that at some layer $\ell' \leq L$ it must have been the case that $\langle x_{t,j'}^1, y_{t}^1 - \sum_{j\in \mathcal{C}^{\ell'}_t} x_{t,j}^1  \rangle \geq 3/4$ for all $t$, otherwise $j'$ can not be added to $\mathcal{C}^{\ell'}_i$ as it is not consistent with $f^\star$ on some round $t$ and so it would not be part of $\mathcal{C}^{L}_i$ as $\mathcal{C}^{\ell'}_i \subseteq \mathcal{C}^{L}_i$. This is now a contradiction as it implies
\begin{align*}
   \frac{3}{4} \leq \langle x_{t,j'}^1, y_{t}^1 - \sum_{j\in \mathcal{C}^{\ell'}_t} x_{t,j}^1 \rangle \leq \langle x_{t,j'}^1, y_{t}^1 \rangle + \frac{s}{2\tau} \leq \frac{1}{2},
\end{align*}
where the second inequality follows from Assumption~\ref{assm:orth} as $\langle x_{i,j'}^1, x_{i,j}^1\rangle > -\frac{1}{2s}$. This shows that we can never add a coordinate which is not similar to some coordinate in $f^\star$ across all the examples till $i$. 

We show that the map is injective as follows. Let $j_{\ell_0}$ some coordinate for which we have already established the mapping $j_{\ell_0} \to j \in f^\star$ at layer $\ell_0$. Consider another candidate $j_{\ell_1}$, for $\ell_1 > \ell_0$ such that $|x_{t,j_{\ell_1}} - x_{t,j}| \leq \epsilon$, that is $j_{\ell_1}$ can potentially be mapped to $j$ as well on round $t$. We consider two cases, first for $j'\in f^\star$ s.t. $j'\neq j$ we have $|x_{t,j_{\ell_1}} - x_{t,j'}| > \epsilon$ or $j' \in \cC_t^{\ell_1-1}$ already. In this case we show that $x_{t,j_{\ell_1}}$ is nearly orthogonal to $y^1_{t} - \sum_{j \in \cC_t^{\ell_1-1}}x_{t,j}^1$ so that $x_{t,j_{\ell_1}}$ can not be added at any layer after $x_{t,j'}$ has been added:
\begin{align*}
     \langle x_{t,j_{\ell_1}}^1, y^1_{t} - \sum_{j \in \cC_t^{\ell_1-1}} x_{t, j}^1 \rangle = \sum_{w \in \cC_t^{\ell_1-1}} \langle x_{t,j_{\ell_1}}^1, x_{t, w}^1 \rangle \leq \frac{s}{\tau},
\end{align*}
where the last inequality follows as before together with the assumption $|x_{t,j_{\ell_1}} - x_{t,j'}| > \epsilon$ outside of $\cC_t^{\ell_1-1}$. Next, if there exists some $j' \in f^\star, j'\not\in \cC_t^{\ell_1-1}$ such that $|x_{t,j_{\ell_1}} - x_{t,j'}| < \epsilon$ we can map $j_{\ell_1} \to j'$ and add $j'$ to $\cC_t^{\ell_1}$ as long as the consistency property holds for all $t'\leq t$. Otherwise, there exists a round $t$ where $|x_{t',j_{\ell_1}} - x_{t',j'}| > \epsilon, \forall j' \in \cC_t^{\ell_1-1}$ and the argument above can be repeated.

Further, we note that the construction can add at least every $j \in f^\star$ to $\mathcal{C}^L_i$ as the following is always satisfied:
\begin{align*}
    \langle x_{i,j}^1, y_i^1 - \sum_{s \in S} x_{i,s}^1 \rangle \geq \frac{3}{4},\forall j \in f^\star, \forall S\subsetneq f^\star,
\end{align*}
unless $S$ contains some coordinate $j'$ such that $|x_{i,j} - x_{i,j'}| \leq \epsilon$ for all $i$. That is, every $j \in f^\star$ is mapped to at least one coordinate in $\cC^L_i$. Taken together, each $j \in \cC^L_i$ is mapped to \emph{exactly one} element of $f^\star$ and each element of $f^\star$ is mapped to some element of $\cC_i^L$. This establishes the claim for the bijection. The second claim of the lemma follows just from the construction of the transformer.
\end{proof}
To use the returned $y_i^\ell$ guaranteed by Lemma~\ref{lem:output_hypothesis} for inference we first modify it in the following way. We add the vector consisting of all $1$s and then apply a relu on each coordinate. The resulting vector now contains a consistent hypothesis in the $y_i^\ell(d+1:m)$. To apply the hypothesis we simply use the construction of the final three layers from the index token task.

\subsection{Proof of Theorem~\ref{thm:sample-complexity-s-sparse-main}}

We treat $f^\star$ and $f_n$ as two subsets of $[m]$ with cardinality $s$. Lemma~\ref{lem:output_hypothesis} implies that for every example $i\in[n]$, there is a bijection $b_n$ between $f_n$ and $f^\star$ which maps any $j \in f^\star$ to a $j' \in f_n$ such that $|x_{i,j} - x_{i,j'}| \leq \epsilon, i\in[n]$. The same argument as in Lemma~\ref{lem:risk_bound} shows the following.
\begin{lemma}
\label{lem:risk_bound_ssparse}
For any $x \in \R^m, j\in f^\star$ let $p_{n,j} = \P(|x_{j} - x_{b_n(j')}| \leq \epsilon)$. Then with probability $1-\delta$ it holds that $p_n \geq 1 - \frac{20s\log(m/\delta)}{3n}$.
\end{lemma}
Using the above lemma we can show the equivalent to the sample complexity bound for the index token task.
\begin{proof}[Proof of Theorem~\ref{thm:sample-complexity-s-sparse-main}]
The same argument as in Theorem~\ref{thm:sample_complexity_itt} can be used to show that for the bijection guaranteed by Lemma~\ref{lem:output_hypothesis} and the setting of $n$ we have $\E[|x_j - x_{b_n(x_j)}|] \leq 2\epsilon$, $\forall j\in f^\star$. This implies the result of the theorem as
\begin{align*}
    \E[|f_n(x) - f^\star(x)|] = \E[|\sum_{j\in f^\star} x_{b_n(j)} - x_{j}|] \leq \sum_{j\in f^\star}\E[|x_{b_n(j)} - x_{j}|] \leq 2s\epsilon.
\end{align*}
Redefining $\epsilon \rightarrow \epsilon/s$ completes the proof. 
\end{proof}

\section{Vector $1$-sparse regression task}
\label{app:vector_task}
We now quickly discuss how to solve the vector version of the $1$-sparse regression task, where the transformer's input is a sequence of examples $(x_i, y_i)_{i\in [n]}$, however, now $x_i \in \R^m$ is a single token, rather than being split into $m$ tokens. The idea is to learn each bit of a consistent hypothesis sequentially using a total of $O(\log(m))$ attention layers. To do so we focus on recovering learning a consistent hypothesis for example $i$ as done in the first attention layer in the $1$-sparse token task. The remainder of the construction follows the ideas from the $1$-sparse token task.

\paragraph{First attention layer.}
Unlike in the $1$-sparse tokenized regression task, we can not represent a single hypothesis by the respective token (even though it does still correspond to a coordinate in $x$). Instead we assume that the value vector $v_{i,1}^1$, for $x_i$, in the first layer, contains 0 in its first $m$ coordinates and the following vector $\beta_{i,1}^1 \in \R^m$ in the next $m$ coordinates
\begin{align*}
    \beta_{i,1}^1(j) = \mathbf{1}(\text{bit $1$ of $j$ equals 1}).
\end{align*}
The value vector $v_{i,2}^1$ for $y_i$ is constructed similarly, with the first $m$ coordinates equal to $0$ again and the second $m$ coordinates equaling $\beta_{i,2}^1 \in\R^m$ which is the complement of $\beta_{i,1}^1$ in $\{0,1\}^m$.
The embeddings in the first layer are as follows. $x_{i}^1 \in \R^{(d+1)m}$ contains the embedding of $x_{i,j}^1$ from Assumption~\ref{assm:orth} in coordinates $x_{i,1}^1(d(j-1) + 1: dj)$. The remaining $m$ coordinated are all set to $1$. $y_{i}^1 \in \R^{(d+1)m}$ is constructed similarly, where the first $dm$ coordinates contain the embedding of $y_i$ from Assumption~\ref{assm:orth}, repeated $d$ times. The last $m$ coordinates equal the last $m$ coordinates of $v_{i,2}^1$, that is $y_i^1(dm+1: (d+1)m) = v_{i,2}^1(m+1:2m)$.
The query and key matrices $Q^1, K^1$ now implement the following linear operation:
\begin{align*}
    \langle y_{i}^1, x_{i}^1 \rangle_{Q^1(K^1)^\top} &= \gamma\sum_{j=1}^m \beta_{i,1}^1(j)\langle y_i^1(d(j-1)+1:dj), x_i^1(d(j-1)+1:dj)\rangle,\\
    \langle y_{i}^1, y_{i}^1 \rangle_{Q^1(K^1)^\top} &= \frac{\gamma}{2}\sum_{j=1}^m \beta_{i,1}^1(j)\langle y_i^1(d(j-1)+1:dj), y_i^1(d(j-1)+1:dj)\rangle.
\end{align*}
This is implemented in the following way, the query matrix $Q$ is a diagonal matrix with $Q(d(j-1) + 1: dj) = \beta_{i,1}^1(j)I_{d\times d}$. In the above $\gamma= \Theta(\log(m/\epsilon))$ is a threshold parameter which will turn the softmax into an approximate max.
We do not specify the inner product $\langle x_i^1, \cdot \rangle_{Q^1(K^1)^\top}$ as the second layer embedding $x_i^2$ will be independent of the first layer.
\begin{lemma}
\label{lem:first_bit}
    The inner product $\langle y_{i}^1, x_{i}^1 \rangle_{Q^1(K^1)^\top} \geq \gamma(1 - \frac{m}{\tau})$ iff there exists at least one consistent with $f^\star$ hypothesis with first bit equal to $1$. Further, if there is no such hypothesis then $\langle y_{i}^1, x_{i}^1 \rangle_{Q^1(K^1)^\top} \leq \gamma\frac{m}{\tau}$.
\end{lemma}
\begin{proof}
 Using the definition of the embeddings we have
 \begin{align*}
     \langle y_{i}^1, x_{i}^1 \rangle_{Q^1(K^1)^\top} = \gamma\sum_{j=1}^m \beta_{i,1}^1(j)\langle \bar y_{i}^1, \bar x_{i,j}^1 \rangle,
 \end{align*}
 where $\bar x_{i,j}^1, \bar y_{i}^1 \in \R^d$ are the embeddings from the $1$-sparse task. From Assumption~\ref{assm:orth} we have that $\langle \bar y_{i}^1, \bar x_{i,j}^1 \rangle \geq 1$ if $j = f^\star$, $\inner{y_i^1}{x_{i,j}^1} \geq 0$ if $j$ is some other coordinate consistent with $f^\star$ and $\beta_{i,j}^1(j) = 1$ iff the first bit of $j$ equals $1$. Hence we get an inner product of at least $\gamma$ from $f^\star$, at least $0$ from any other consistent coordinate, and at least $-1/\tau$ from any inconsistent coordinates. This implies the first claim of the lemma. For the second part we note that if there is no consistent hypothesis with first bit equal to $1$ then $\langle \bar y_{i}^1, \bar x_{i,j}^1 \rangle \leq \frac{1}{\tau}$ according to Assumption~\ref{assm:orth}.
\end{proof}
To keep the argument clean, we assume that the softmax acts as an argmax. As we have pointed out, this can be achieved up to $\epsilon$ when setting $\gamma = \Theta(\log(m/\epsilon))$. The output for the $i$-th answer token, $o_{i,2}^1$, now contains in its last $m$ coordinates an indicator of which hypotheses are consistent, when restricted to the value of the first bit. In particular, if there exists a consistent hypothesis then $o_{i,2}^1(m+1:2m) = \beta^1_{i,1}$ and otherwise $o_{i,2}^1(m+1:2m) = \beta^1_{i,2}$. 
\begin{lemma}
\label{lem:coord_selection}
Let $z \in \R^m$ be some vector such that $\|z\|_\infty \leq c < \infty$ and $\beta \in \{0,1\}^m$. Let $z \odot \beta$ denote the element-wise product of the two vectors. Then the operation $z \odot \beta$ can be implemented by a Relu MLP layer.
\end{lemma}
\begin{proof}
Let $\mathbf{e} \in \R^m$ be the all ones vector. The MLP applies the following operation $\text{Relu}(z + c(\mathbf{e} - \beta) - c\beta) - c(\mathbf{e} - \beta)$.
\end{proof}
Using Lemma~\ref{lem:coord_selection} the MLP acts on $o_{i,2}^1$ by setting $y_{i}^2(d(j-1)+1:dj) := o_{i,2}^1(j)o_{i,2}^1(d(j-1)+1:dj)$, so that the first $dm$ entries of $y_i^2$ only contain coordinates which are consistent with the recovered bit in the first layer.

\paragraph{Second attention layer.}
In this layer we demonstrate how to learn the second bit of a consistent hypothesis for example $i$, conditioned on the first bit contained in $y_{i}^2(m+1:2m)$. The value vectors are defined similarly to the first layer, using 
\begin{align*}
    \beta_{i,1}^2(j) = \mathbf{1}(\text{bit $2$ of $j$ equals 1}),
\end{align*}
and its complement $\beta_{i,2}^2 \in \R^m$. For the embeddings, $x_{i}^2 = x_{i}^1$, and $y_i^2$ is as described above.
Finally we set $K^2=K^1$ and $Q^2$ is defined to act similarly to $Q^1$, however, with respect to $\beta_{i,1}^2$, that is:
\begin{align*}
    \langle y_{i}^2, x_{i}^2 \rangle_{Q^2(K^2)^\top} = \gamma \sum_{j=1}^m \beta_{i,1}^2(j)\langle y_{i}^2(d(j-1)+1:dj), x_i^2(d(j-1)+1:dj) \rangle.
\end{align*}

A result similar to Lemma~\ref{lem:first_bit} can now be shown, where the attention weight $A^2(y_{i}, x_i) \approx 1$ if there exists a consistent hypothesis with first bit set according to $o_{i,2}^1$ and second bit equal to $1$, otherwise $A^2(y_i, y_i) \approx 1$ and there exists a consistent hypothesis with first bit set according to $o_{i,2}^1$ and second bit equal to $0$. Finally, we describe how the MLP is applied.
First, we add $o_{i,2}^1 + y_{i}^2(dm+1: (d+1)m)$ using the skip connection. This results in the following (assuming a max, instead of a soft-max).
\begin{lemma}
\label{lem:two_bits}
The $j$-th coordinate of $o_{i,2}^1 + y_{i}^2(dm+1: (d+1)m)$ satisfies $o_{i,2}^1 + y_{i}^2(dm+1: (d+1)m) \geq 2$ iff the $j$-th hypothesis is consistent with $f^\star$ on the $i$-th example.
\end{lemma}
\begin{proof}
 WLOG assume that $A^1(y_i, x_i) \approx 1$ and $A^2(y_i, y_i) \approx 1$
, so that the inner product in the second layer has shown that there exists a consistent hypothesis with first two bits equal to $10$. From the construction it holds that $o_{i,2}^1$ indexes all hypotheses with second bit set to $0$. Further, $o_{i,2}^1$ indexes all hypothesis with first bit set to $1$, and so under the assumption $o_{i,2}^1 + y_{i}^2(dm+1: (d+1)m) = o_{i,2}^1 + o_{i,2}^2$ will have $j$-th coordinate greater than $2$ if the $j$-th hypothesis is consistent on the $i$-th example and has first two bits equal to $10$.
\end{proof}
We apply the following operation to $o_{i,2}^2 + y_{i}^2(dm+1: (d+1)m)$. First we subtract some threshold $2 > c > 1$. Then we clip each negative coordinate to $0$ and each positive coordinate to $1$. Lemma~\ref{lem:two_bits} implies that the resulting vector indexes exactly all consistent hypotheses with first and second bit set according to $o_{i,2}^1$ and $o_{i,2}^2$ respectively. Let this vector be $y_i^3(dm+1: (d+1)m) \in \R^m$. We now want to apply $y_i^3(dm+1: (d+1)m) \in \R^m$ to $y_i^1(1: dm)$, similarly to how the first layer MLP applied the consistent hypothesis to $y_i^1(1: dm)$ as well. To do so we require an extra MLP layer. Note that this can be achieved by adding another attention layer to carry out this operation.

\paragraph{Further layers.} Replicating the construction for the second layer, but focusing on the $b$-th bit of the hypothesis we can show the following
\begin{lemma}
After $M = O(\log(m))$ layers it holds that $y_{i}^M(dm+1: (d+1)m)(j) = \mathbf{1}(j\text{ is consistent with }f^\star)$.
\end{lemma}
One can now use the same type of construction as in the index token task to learn a hypothesis which is consistent on all $i$ examples seen so far and further use this hypothesis to do inference. The sample complexity bound for this approach are similar to the one in Theorem~\ref{thm:sample_complexity_itt}. We also note that this construction can be extended to handle the $s$-sparse index vector task as well, but we will not go into details as the constructions required should not demonstrate any new ideas.

\section{Experiments}
\label{app:experiments}
\subsection{$1$-sparse tokenized regression experiments}
\label{app:token_task_experiments}
We experiment with two settings for 1-sparse tokenized regression. In both settings the dimensionality of the problem is $m=5$, that is each example $x_i$ consists of $5$ tokens $(x_{i,1},\ldots, x_{i,5})$ together with the answer token $y_i$. The transformer architecture is the same for both tasks. We use $8$ attention layers, with masking future tokens, that is the only non-zero attention weights are $A^i(x_{i,j}, x_{i',j'})$ for $i\geq i', j\geq j'$. Each attention layer is follows by a layer-norm normalization and an MLP layer with $\gelu$ activation. Further, skip connections are used between the input to the attention layers and the output of the layer-norm and MLP layers. The hidden size for the embeddings is $d=128$, and we use a single attention head per layer. Positional embeddings are learned. Predictions are done by a final MLP layer mapping the $d$-dimensional embeddings to a scalar. The training for both settings uses the same hyper-parameters and optimizer. We use Adam as optimizer with the schedule used in \citep{akyurek2022learning} and initial step-size set to $1e-4$. Initializing the network parameters also follows \cite{akyurek2022learning}.

Training in both settings proceeds by generating example sequences $(x_i, y_i)_{i\in [n]}$, by first selecting a fixed hypothesis $f^\star$, sampled uniformly at random from $[m]$ and then sampling $x_i$'s i.i.d. from fixed distributions which we describe momentarily. The sequence for a single pre-training iteration is then $(x_i, f^\star(x_i))_{i\in [n]}$. Pre-training proceeds in mini-batches of size $64$, that is each mini-batch has $64$ sequences $(x_i, f^\star(x_i))_{i\in [n]}$ sampled independently as we described above. We use mean-squared loss over the sequence for pre-training in the following way:
\begin{align*}
    \mathcal{L}((x_i, f^\star(x_i))_{i\in [n]}; \theta) = \frac{1}{n}\sum_{i=1}^n (x_{i, 5}^8 - f^\star(x_i))^2,
\end{align*}
where $\theta$ denotes the parameters of the transformer and $x_{i,j}^8$ is the output of the final MLP layer of the transformer (in accordance with our index token task notation), that is we use the transformers output after seeing the $5$-th token in each example as the prediction, and the loss is taken as the squared difference between this prediction and the $6$-th token in each example. Finally, we note that while $f^\star$ is fixed for a sequence $(x_i, y_i)_{i\in[n]}$, we sample a fresh $f^\star$ for each new sequence in the mini-batch. This setup is identical to both~\citet{gargcan} and \citet{akyurek2022learning}. We train for $8000$ epochs, where each epoch consists of $100$ iterations, each one on a mini-batch of size $64$.

The two settings we consider are in terms of the distribution over $(x_i, y_i)_{i\in [n]}$. In the first setting $x_i \sim \mathcal{N}(0, I_{5\times 5})$ and in the second setting $x_{i} \sim Unif(\{+1, -1\}^5)$. These settings are complementary to each other in the following way. In the Gaussian setting it is possible to learn the hypothesis $f^\star$ after a single ICL example almost surely. In the uniform over $\{-1,1\}^m$ setting, which refer to as the Rademacher setting, one needs to see $\Omega(\log(m))$ examples before $f^\star$ can be identified with high probability.

We plot the squared error at every example for a given sequence, attention at the last layer and attention at layer $6$. Plots are averaged over the mini-batch of size $64$. For averaging we fix the same $f^\star$ over the full mini-batch. Results can be found in Figure~\ref{fig:gg}, Figure~\ref{fig:gh}, Figure~\ref{fig:mh}.
\begin{figure}
\centering
\begin{subfigure}[b]{.32\textwidth}
    \includegraphics[width=\textwidth]{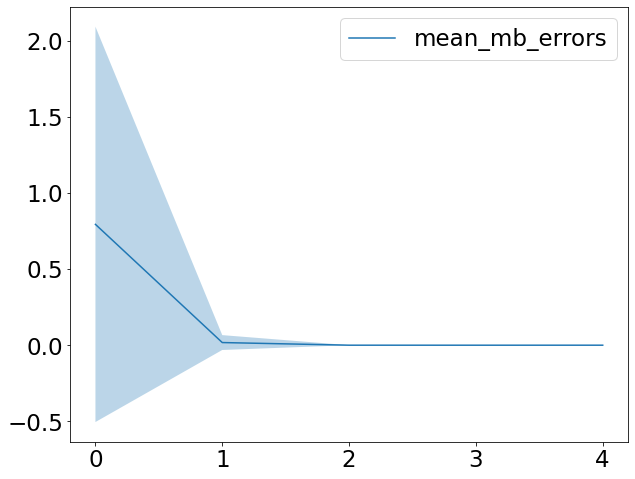}
    \caption{Loss}
\end{subfigure}
\begin{subfigure}[b]{0.32\textwidth}
    \includegraphics[width=\textwidth]{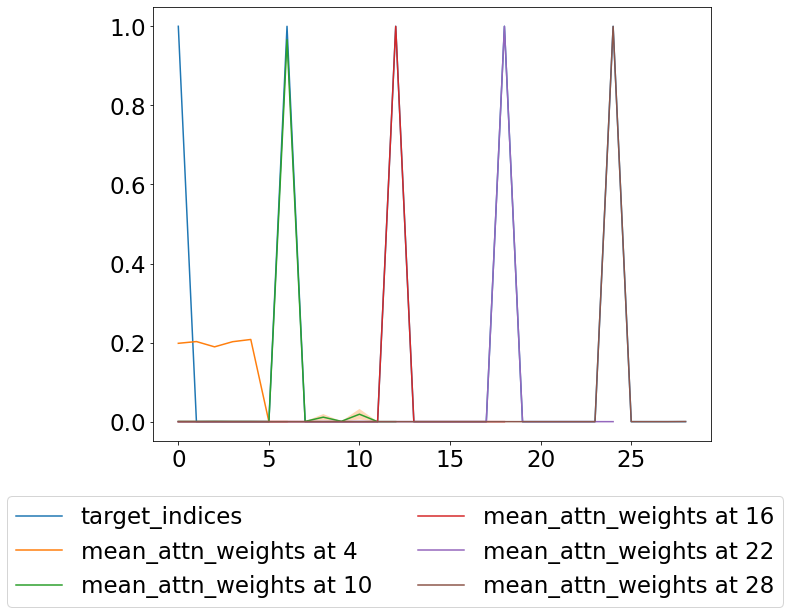}
    \caption{Attention at 8}
\end{subfigure}
\begin{subfigure}[b]{0.32\textwidth}
    \includegraphics[width=\textwidth]{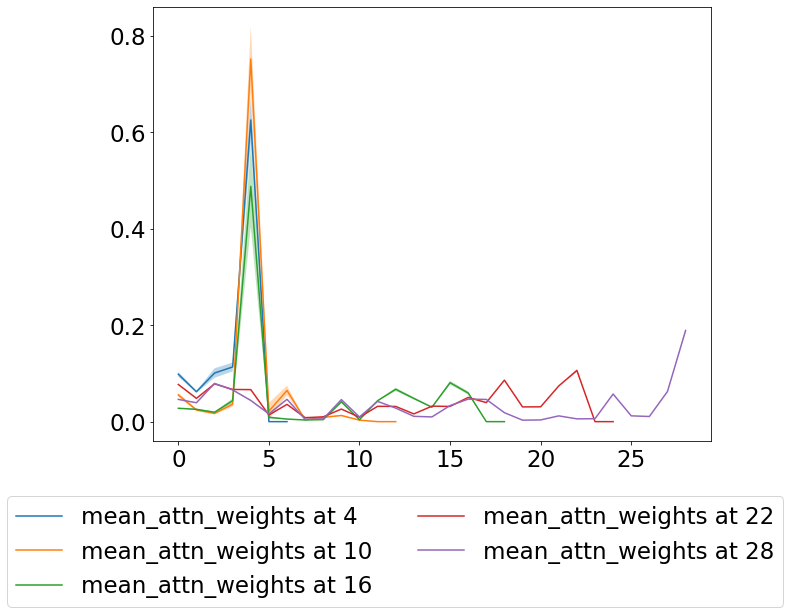}
    \caption{Attention at 6}
\end{subfigure}
    \caption{Train Gaussian, inference Gaussian.}
    \label{fig:gg}
\end{figure}

\begin{figure}
\centering
\begin{subfigure}[b]{.32\textwidth}
    \includegraphics[width=\textwidth]{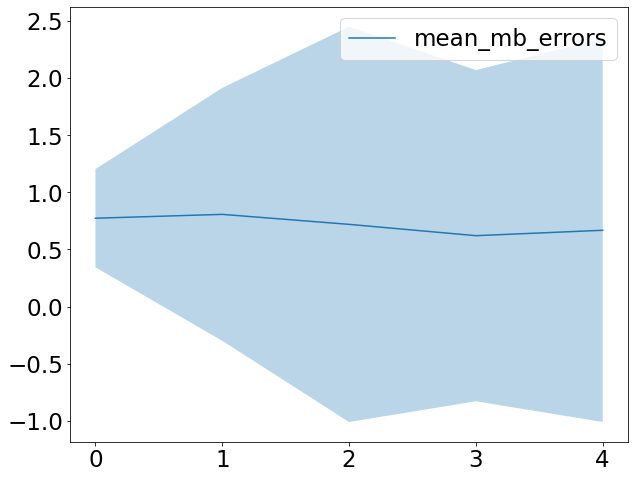}
    \caption{Loss}
\end{subfigure}
\begin{subfigure}[b]{0.32\textwidth}
    \includegraphics[width=\textwidth]{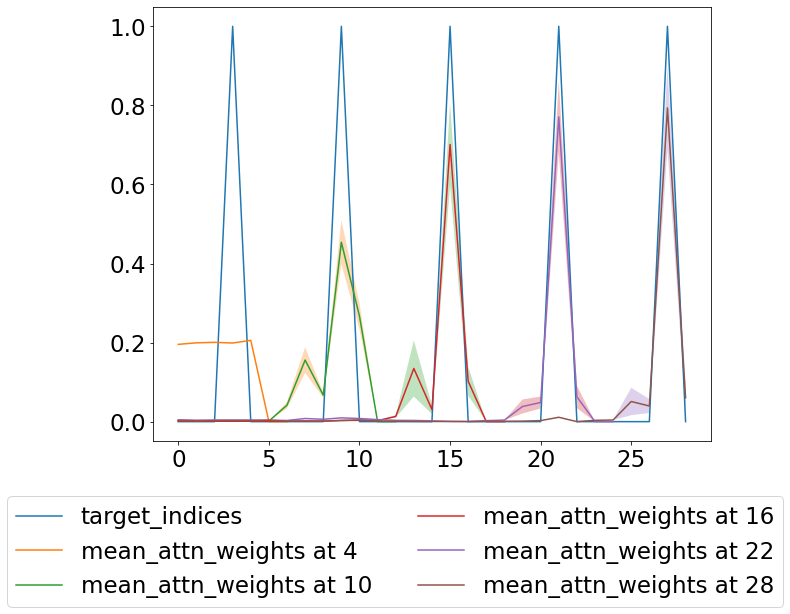}
    \caption{Attention at 8}
\end{subfigure}
\begin{subfigure}[b]{0.32\textwidth}
    \includegraphics[width=\textwidth]{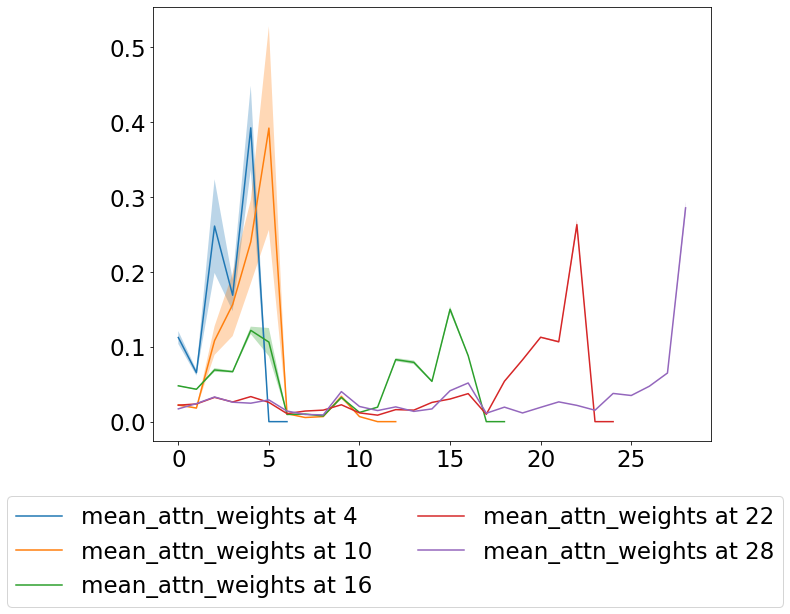}
    \caption{Attention at 6}
\end{subfigure}
    \caption{Train Gaussian, inference Rademacher.}
    \label{fig:gh}
\end{figure}

\begin{figure}
\centering
\begin{subfigure}[b]{.32\textwidth}
    \includegraphics[width=\textwidth]{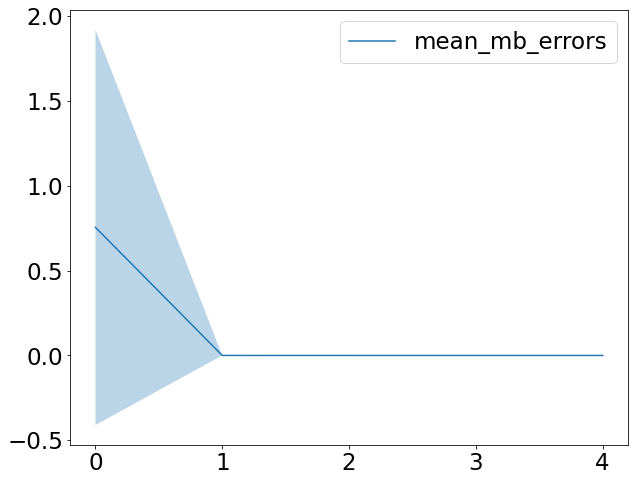}
    \caption{Loss}
\end{subfigure}
\begin{subfigure}[b]{0.32\textwidth}
    \includegraphics[width=\textwidth]{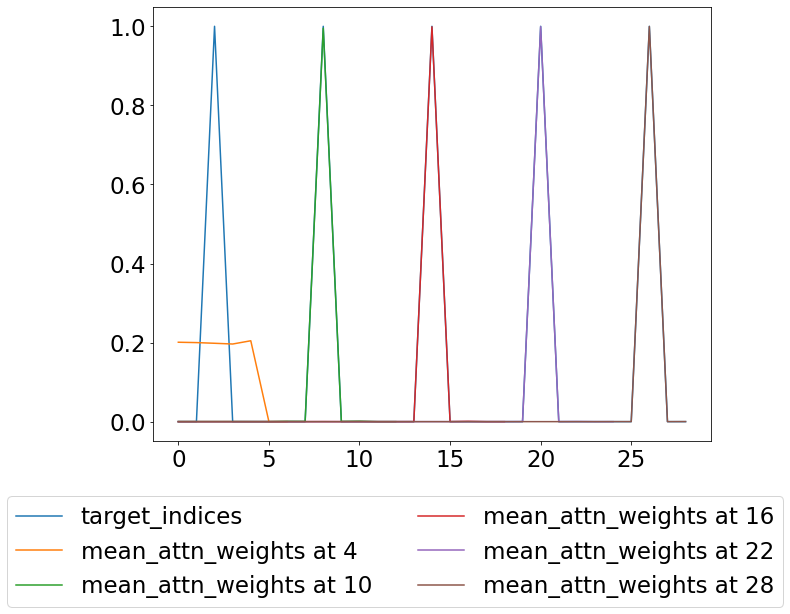}
    \caption{Attention at 8}
\end{subfigure}
\begin{subfigure}[b]{0.32\textwidth}
    \includegraphics[width=\textwidth]{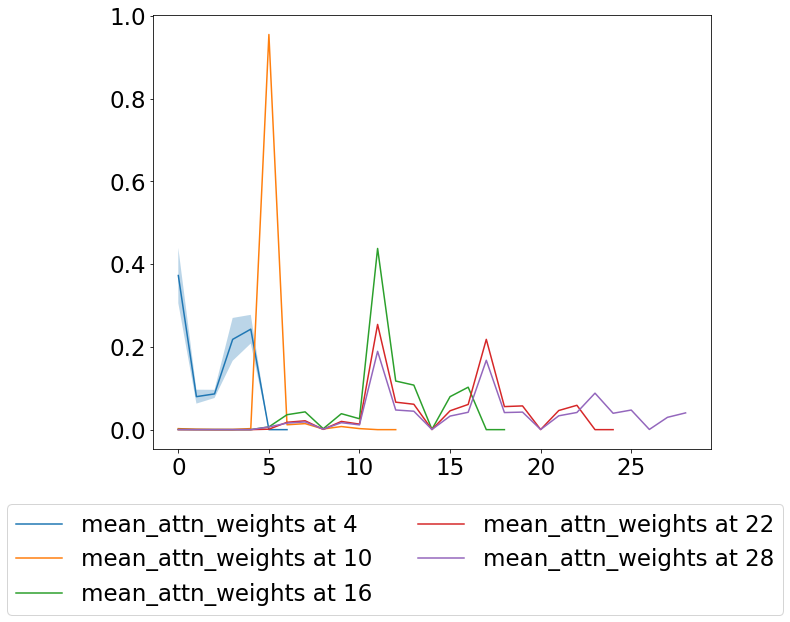}
    \caption{Attention at 6}
\end{subfigure}
    \caption{Train Rademacher, inference Gaussian.}
    \label{fig:hg}
\end{figure}

\begin{figure}
\centering
\begin{subfigure}[b]{.32\textwidth}
    \includegraphics[width=\textwidth]{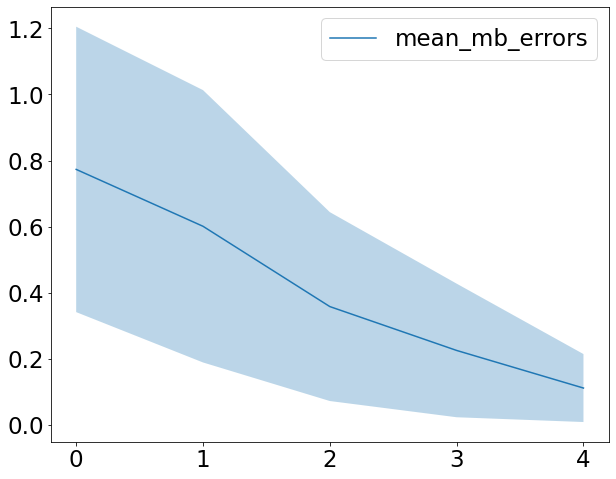}
    \caption{Loss}
\end{subfigure}
\begin{subfigure}[b]{0.32\textwidth}
    \includegraphics[width=\textwidth]{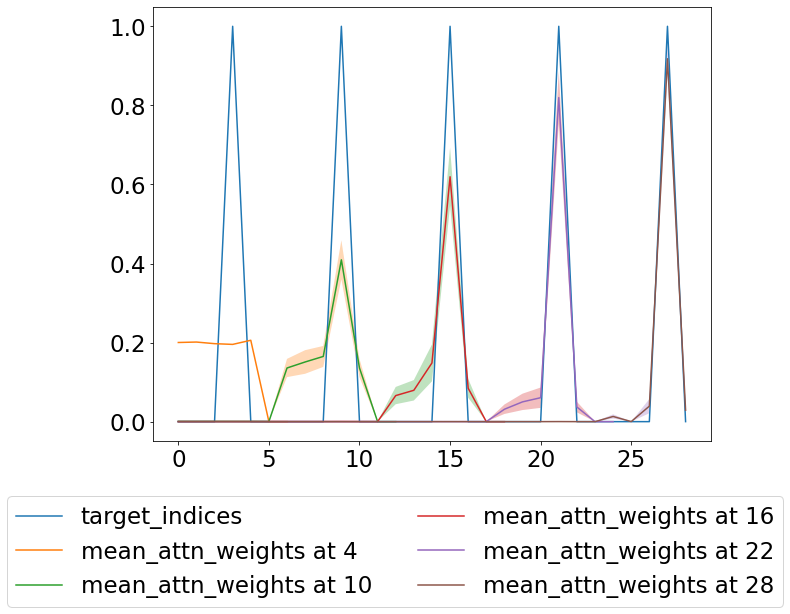}
    \caption{Attention at 8}
\end{subfigure}
\begin{subfigure}[b]{0.32\textwidth}
    \includegraphics[width=\textwidth]{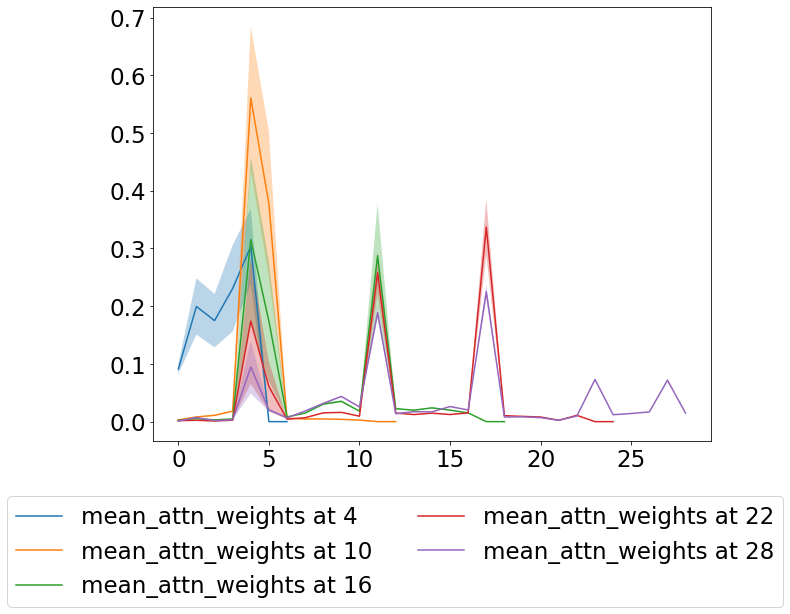}
    \caption{Attention at 6}
\end{subfigure}
    \caption{Train Rademacher, inference Rademacher.}
    \label{fig:hh}
\end{figure}
The transformer pre-trained on the Rademacher task only, exhibits very similar properties to the mixed model discussed in Section~\ref{sec:experiments_main}. Perhaps, surprisingly, the model is able to achieve the same performance on the Gaussian inference task as the mixed model, even though it has never seen Gaussian examples.

The transformer pre-trained on the Gaussian only task, still retains the ability to learn from a single example as demonstrated by Figure~\ref{fig:gg}. However,the attention at layer 6 are less interpretable compared to the mixed model and the Rademacher only model. The attention weights in the last layer retain the nice properties from the other two models. The Gaussian model, however, performs poorly on the Rademacher task as seen in Figure~\ref{fig:gh}. We note that the attention weights at the last layer still behave similarly to the attention weights of the mixed model and the Rademacher model, suggesting that the Gaussian model can still distinguish $f^\star$. We conjecture that the reason for the poor performance is due to how the learned hypothesis is applied to examples for inference. In particular, we expect that the Gaussian model, during pre-training, has learned to apply the inferred hypothesis after the first example, however, this would be detrimental for the Rademacher setting, as it is very unlikely that $f^\star$ is identifiable after only a single example.

\begin{figure}
    \centering
    \includegraphics[width=0.5\textwidth]{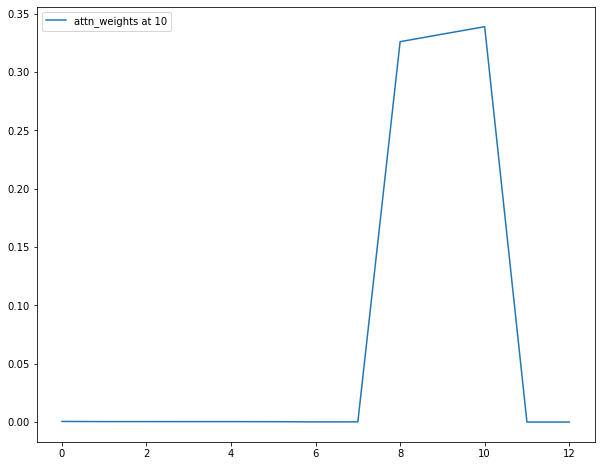}
    \caption{Rademacher task attention spread}
    \label{fig:attn_spread}
\end{figure}

Finally in Figure~\ref{fig:attn_spread} we show the behavior of the mixed model on a single Rademacher sequence. The first $6$ elements of the sequence from the figure are $x_{1,1} = 1, x_{1,2} = 2, x_{1,3} = -1, x_{1,4}=-1, x_{1,5}= -1, y_1=-1$. At inference time for the second example, we show the attention weights for example $x_{1,4}$, which is token $z_{10}$ in the sequence, spreads its attention uniformly on all consistent hypothesis $j \in \{3,4,5\}$ corresponding to tokens $z_8, z_9, z_{10}$. This is again consistent with our construction for inference. In our experiments we have observed that the attention is put on $f^\star$ at the earliest example $i$ where the identification is possible, and this is why the averaged attention plots at layer $8$ are peaked, with some variance, at $f^\star$.

\subsection{Segmentation}
\label{app:segmentation_experiments}
We make the following empirical observation: the performance of ICL is sensitive to the choice of delimiter. In Figure~\ref{fig:delimiter_sensitivity} we show the quality of ICL using OpenAI's GPT-3 model (known as \textit{text-davinci-003}) on a family of relational tasks and using a range of different delimiters. The tasks in question are relational tasks, usually covering a type of ``trivia'' question. But we vary the two delimiters and consider performance of the completion. Below are three example queries we provide to the model, and we consider the answer \textit{correct} if the correct answer occurs within the first 3 tokens of the response.
{\small
\begin{verbatim} 
// scientist year of death
Albert Einstein => 1955 \n Isaac Newton => 1727 \n Johannes Kepler => ______ 
// famous actor year of birth
Leonardo DiCaprio is 1974 but Meryl Streep is 1949 but Dustin Hoffman is ______ 
// baseball team last won world series
Houston Astros is not 2017 / St. Louis Cardinals is not 2011 / Boston Red Sox is not ______ 
\end{verbatim}
}
\begin{figure}
    \centering
    
    \begin{subfigure}[b]{0.45\textwidth}
        \includegraphics[width=\textwidth]{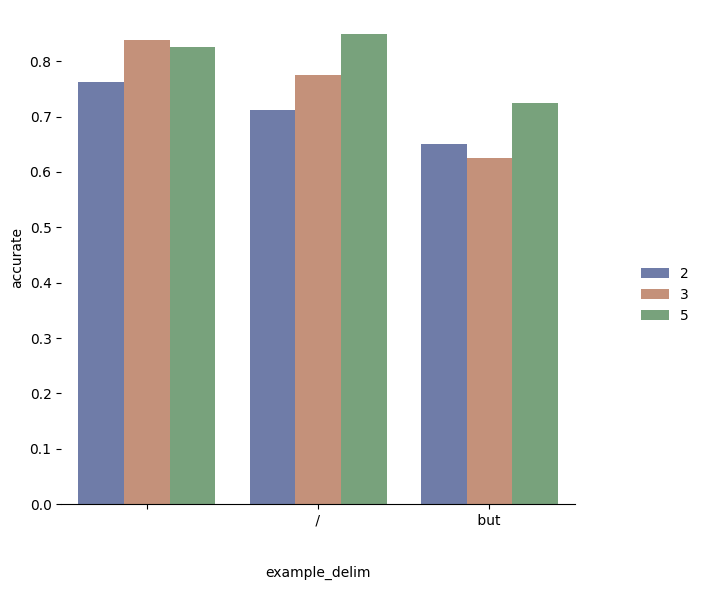}
        \label{fig:image1}
    \end{subfigure}
    \hfill
    \begin{subfigure}[b]{0.45\textwidth}
        \includegraphics[width=\textwidth]{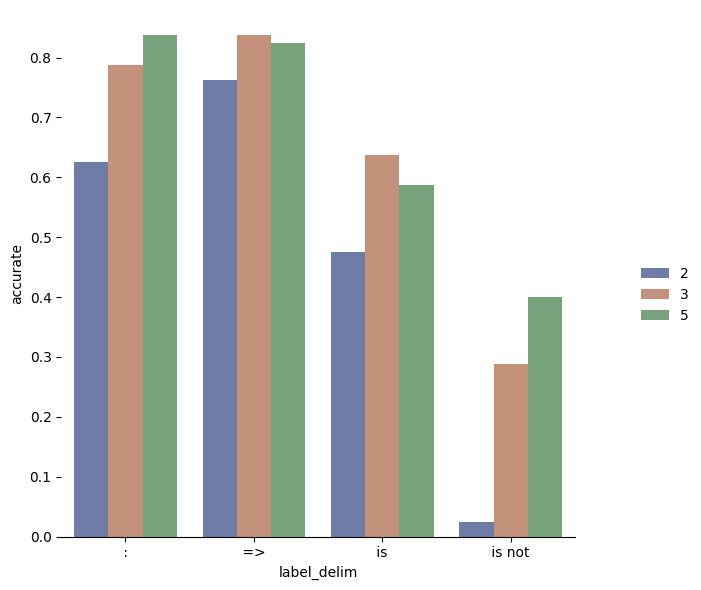}
        \label{fig:image2}
    \end{subfigure}
    
    \caption{The accuracy of ICL over a range of tasks when we vary the choice of delimiter. On the left figure, we vary the delimiter used to separate examples among $\{\texttt{\textbackslash n}, \texttt{/}, \texttt{but}\}$, and on the right we vary the delimiter used to separate $x$ from $y$ among $\{\texttt{:}, \texttt{=>}, \texttt{is}, \texttt{is not}\}$. The performance is computed across four association tasks, we run each task 10 times (across different ``training'' example sets), and for all the example delimiter tasks we use the label delimiter \texttt{:}, and for all the label delimiter tasks we use the example delimiter \texttt{\textbackslash n}.}
    \label{fig:delimiter_sensitivity}
\end{figure}